\documentclass[preprint,11pt]{article}
\usepackage{fullpage}
\usepackage{amssymb,amsfonts,amsmath,amscd,mathrsfs,bm,amsthm}
\usepackage{graphicx,float,psfrag,epsfig,amssymb}
\usepackage[usenames,dvipsnames,svgnames,table]{xcolor}
\definecolor{darkgreen}{rgb}{0.0,0,0.9}
\usepackage[pagebackref,letterpaper=true,colorlinks=true,pdfpagemode=none,citecolor=OliveGreen,linkcolor=BrickRed,urlcolor=BrickRed,pdfstartview=FitH]{hyperref}
\usepackage{hyperref}
\usepackage[tight]{subfigure}
\usepackage{caption}

\let\chapter\section
\usepackage[linesnumbered,lined,boxed,commentsnumbered]{algorithm2e}
\usepackage[noend]{algorithmic}

\def\ds@whichfont{dsrom}
\relax

\DeclareMathAlphabet{\mathds}{U}{\ds@whichfont}{m}{n}

\addtocontents{toc}{\protect\setcounter{tocdepth}{2}}

\def\hbeta{\widehat{\beta}}
\def\tX{\widetilde{X}}

\def\tv{\tilde{v}}

\def\tX{\widetilde{X}}
\def\tx{\widetilde{x}}

\def\cF{{\cal F}}
\def\cA{{\cal A}}

\def\cH{{\cal H}}
\def\bcH{\bar{\cal H}}

\def\cE{{\cal E}}
\def\cD{{\cal D}}

\def\tp{{\tilde{p}}}

\def\bz{\mathbf{z}}

\def\tSigma{\tilde{\Sigma}}

\def\dis{{\sf d}}

\def\naturals{{\mathbb N}}

\def\reals{{\mathbb R}}

\def\prob{{\mathbb P}}
\def\E{{\mathbb E}}
\def\Var{{\rm Var}}

\def\tC{\widetilde{C}}

\def\cB{{\cal B}}
\def\tZ{\tilde{Z}}

\def\L0{{L_i}}

\def\de{{\rm d}}
\def\<{\langle}
\def\>{\rangle}

\def\hth{\widehat{\theta}}

\def\hSigma{\widehat{\Sigma}}

\def\supp{{\rm supp}}

\def\F{{\sf F}}
\def\ind{{\mathbb I}}

\def\tz{\tilde{z}}

\def\F{{\sf F}}
\def\normal{{\sf N}}

\def\tX{{\tilde{X}}}
\def\ttheta{{\tilde{\theta}}}

\def\sT{{\sf T}}

\def\id{{\rm I}}

\def\v*{v_i}
\def\T*{T_i}

\def\u*{u_i}
\def\F*{F_i}

\definecolor{olivegreen}{rgb}{0,0.6,0.4}

\def\H{\mathcal{H}}

\def\cH{{\mathcal{H}}}

\def\htheta{\widehat{\theta}}

\def\tx{\tilde{x}}

\def\cL{\mathcal{L}}
\def\tilth{\tilde{\theta}}

\def\th{{\theta}}
\def\hth{{\widehat{\theta}}}
\def\tth{{\theta_0}}

\def\cX{\mathcal{X}}

\def\Reg{{\sf Regret}}
\def\popt{p^*}

\def\pX{\mathbb{P}_X}

\def\l1u{W}
\def\f{\psi}
\def\p{\mu}
\def\hp{\widehat{\mu}}
\def\tila{\tilde{\alpha}}

\def\tSigma{\widetilde{\Sigma}}
\def\tX{\widetilde{X}}
\def\tilp{\tilde{\p}}

\newcommand{\ajcomment}[1]{}

\makeatletter
\newcommand{\labitem}[2]{%
\def\@itemlabel{\text{#1}}
\item
\def\@currentlabel{#1}\label{#2}}
\makeatother

\DeclareMathAlphabet{\mathpzc}{OT1}{pzc}{m}{it}

\DeclareMathAlphabet{\mathpzc}{OT1}{pzc}{m}{it}
\newtheorem{proposition}{Proposition}[section]
\newtheorem{lemma}[proposition]{Lemma}
\newtheorem{definition}[proposition]{Definition}
\newtheorem{coro}[proposition]{Corollary}
\newtheorem{theorem}[proposition]{Theorem}

\theoremstyle{definition}

\newtheorem{assumption}[proposition]{Assumption}

\title{Dynamic Pricing in High-dimensions}

\author{Adel Javanmard\thanks{Department of Data Sciences and Operations, Marshall School of Business, University of Southern California
Email: {\{ajavanma,nazerzad\}@usc.edu}
}\;\;\;\;
Hamid Nazerzadeh${}^{\ast}$
}

\begin{document}

\maketitle

\begin{abstract}
We study the pricing problem faced by a firm that sells a large number of products, described via a wide range of features, to customers that arrive over time.
Customers independently make purchasing decisions according to a general choice model that includes products features and customers' characteristics, encoded as $d$-dimensional numerical vectors, as well as the price offered.
The parameters of the choice model are a priori unknown to the firm, but can be learned as the (binary-valued) sales data accrues over time.  The firm's objective is to minimize the regret, i.e., the expected revenue loss against a clairvoyant policy that knows the parameters of the choice model in advance, and always offers the revenue-maximizing price.  
This setting is motivated in part by the prevalence of online marketplaces that allow for real-time pricing.

We assume a structured choice model, parameters of which depend on $s_0$ out of the $d$ product features. 
We propose a dynamic policy, called Regularized Maximum Likelihood Pricing (RMLP) that leverages the (sparsity) structure of the high-dimensional model and obtains a logarithmic regret in $T$.
More specifically, the regret of our algorithm is of $O(s_0 \log d \cdot \log T)$. Furthermore, we show that no policy can obtain regret better than $O(s_0 (\log d + \log T))$.

\end{abstract}


\section{Introduction} \label{sec:intro}
%

A central challenge in revenue management is determining the optimal pricing policy when there is uncertainty about customers' willingness to pay. Due to its importance, this problem has been studied extensively~\cite{kleinberg2003value,BesbesZ09,badanidiyuru2013bandits,WangDY14,broder2012dynamic,keskin2014dynamic,debBoerZwart2014,cohen2016feature}. 
Most of these models are built around the following classic setting: customers arrive over time; the seller posts a price for each customer; if the customer's valuation is above the posted price, a sale occurs and the seller collects a revenue in the amount of the posted price; otherwise, no sale occurs and no revenue is generated. Based on this and the previous feedbacks, the seller updates the posted price. Therefore, the seller is involved in the realm of exploration-exploitation as he needs to choose between learning about the valuations and exploiting what has been learned so far to collect revenue.

In this work, we consider a setting with a large number of products which are defined via a wide range of features. The valuations are given by $v(\theta,x)$ with $x$ being the (observable) feature vectors of products and $\theta_0$ {representing the customer's characteristics (true parameters of the choice model, which is initially unknown to the seller, cf.~\cite{amin2014repeated,cohen2016feature}.) }. An important special case of this setting is the linear model in which $$v(\theta,x) = \theta_0\cdot x+ \alpha_0 + z\,,$$ where $z$ captures the idiosyncratic noise in valuations and {$\alpha_0$ is an unknown intercept}.

Our setting is motivated in part by applications in online marketplaces. For instance, a company such as Airbnb recommends prices to hosts based on many features including the space (number of rooms, beds, bathrooms, etc.), amenities (AC, WiFi, washer, parking, etc.), the location (accessibility to public transportation, walk score of the neighborhood, etc.), house rules (pet-friendly, non-smoking, etc.), as well as the prediction of the demand which itself depends on many factors including the date, events in the area, availability and prices of near-by hotels, etc.~\cite{Airbnb}. Therefore, the vector describing each property can have hundreds of features. 
Another important application comes from online advertising. Online publishers set the (reserve) price   of ads based on many features including user's demographic, browsing history, the context of the webpage, the size and location of the ad on the page, etc.

In this work, we propose {\em Regularized Maximum Likelihood Pricing} (RMLP) policy for dynamic pricing in high-dimensional environments.
As suggested by its name, the policy uses maximum likelihood method to estimate the true parameters of the choice model. In addition, using an ($\ell_1$-norm) regularizer, our policy exploits the structure of the optimal solution; namely, the performance of the RMLP policy significantly improves if the valuations are essentially determined by a small subset of features. 
More formally, the difference between the revenue obtained by our policy and the benchmark policy that knows in advance the true parameters of the {choice model}, $\p_0= (\tth,\alpha_0)$, is bounded by $O\big(s_0 \log d \cdot \log T\big)$, where $T$, $d$, and $s_0$ respectively denote the length of the horizon, number of the features, and sparsity (i.e., number of non-zero elements of $\p_0$). We show that our results are tight up to a logarithmic factor. Namely, no policy can obtain regret better than $O\big(s_0 (\log d + \log T)\big)$.

We point out that our results can be applied to applications where the features' dimensions are larger than the time horizon of interest. A powerful pricing policy for these applications should obtain regret that scales gracefully with the dimension. Note that in general, little can be learned about the model parameters $\p_0$ if $T<d$, because the number of degrees of freedom $d$ exceeds the number of observations $T$, and therefore, any estimator can be arbitrary erroneous. However, when there is prior knowledge about the structure of unknown parameter $\p_0$, (e.g., sparsity), then accurate estimations are attainable even when $T<d$. 

%
\subsection{Organization} The rest of the paper is organized as follows: In the remaining part of the introduction, we discuss how our work is positioned with respect to the literature and highlight our contributions. In Section~\ref{sec:model}, we formally present our model and discuss the technical assumptions and the benchmark policy. The RMLP policy is presented in Section~\ref{sec:pricing_alg}, followed by its analysis in Section~\ref{sec:regret}. We provide in Section~\ref{sec:lower-bound}, a bound on the performance of any dynamic pricing policy that does not know the choice model in advance. In Section~\ref{sec:general}, we generalize the RMLP policy to non-linear valuations functions. The proofs are relegated to the appendix.

\subsection{Related Work}
Our work contributes to literature on dynamic pricing as well as high dimensional statistics. In the following, we briefly overview the work closest to ours in these contexts.
\begin{description}

\item{{\bf Dynamic Pricing and Learning.}}
The literature on dynamic pricing and learning has been growing over the past few years, motivated in part by the advances in big data technology that allow firms to easily collect and utilize information. We briefly discuss some of the recent lines of research in this literature. We refer to \cite{den2015survey} for an excellent survey on this topic.

\begin{itemize}
\item{\emph{Parametric Approach.}} 
A natural approach to capture uncertainty about the customers' valuations is to model the uncertainty using a small number of parameters, and then estimate those parameters using classical statistical methods such as maximum likelihood~\cite{broder2012dynamic,den2013simultaneously,debBoerZwart2014} or least square estimation~\cite{goldenshluger2013linear,keskin2014SCV,bastani2016decision}. 
Our work is similar to this line of work, in that we assume a parametric model for customer's valuations and apply the maximum likelihood method using the randomness of the idiosyncratic noise in valuations. However, the parameter vector $\theta$ is high-dimensional, whose dimension $d$ (that can even exceed the time horizon of interest $T$). We use \emph{regularized} maximum-likelihood in order to promote sparsity structure in the estimated parameter. Further, our pricing policy has an episodic theme which makes the posted prices $p_t$ in each episode independent of the idiosyncratic noise in valuations, $z_t$, in that episode. This is in contrast to other policies based on maximum-likelihood, such as {\sf MLE-GREEDY}~\cite{broder2012dynamic}, or greedy iterative least square (GILS)~\cite{keskin2014SCV,debBoerZwart2014,qiang2016dynamic} that use the entire history of observations to update the estimate for the model parameters at each step.



\item{\emph{Bayesian Approach.}}
 One of the earliest work on Bayesian parametric approach in this context is by \cite{rothschild1974} who consider a Bayesian framework where the firm can choose from two prices with unknown demand and show that (myopic) Bayesian policies may lead to ``incomplete learning." However, carefully designed variations of the myopic policies can (optimally) learn the optimal price~\cite{harrison2012bayesian}; see also~\cite{keller1999optimal,araman2009dynamic,farias2010dynamic,keskin2014dynamic}.

\item{\emph{Non-Parametric models.}} 
An early work in non-parametric setting is by \cite{kleinberg2003value}. They model the dynamic pricing problem as a multi-armed bandit (MAB) where each arm corresponds to a (discretized) posted price. They propose an $O(\sqrt{T})$-algorithm where $T$ is the length of the horizon. 
Similar results have been obtained in more general settings~\cite{badanidiyuru2013bandits,AgrawalDevanur14} including setting with inventory constraints~\cite{BesbesZ09,BabaioffLimitedSupply2012,WangDY14}.

\item{\emph{Feature-based Models.}} 
Recent papers on dynamic pricing consider models with features/covariates. 
\cite{amin2014repeated}, in a model similar to ours, present an algorithm that obtains regret $O(T^{2/3})$; they also study dynamic incentive compatibility in repeated auctions. 
Another closely related work to ours is by \cite{cohen2016feature}. Their model differs from ours in two main aspects: $i)$ their model is deterministic (no idiosyncratic noise) $ii)$ the arrivals (of features vectors) is modeled as adversarial. They propose a clever binary-search approach using the Ellipsoid method which obtains regret of $O(d^2 \log(T/d))$. 
\cite{qiang2016dynamic} study a model where the seller can observe the demand itself, not a binary signal as in our setting. They show that a myopic policy based on least-square estimations can obtain a logarithmic regret.
To the extent of our knowledge, ours is the first work that highlights the role of structure/sparsity in dynamic pricing.
  
\cite{bastani2016decision} study a multi-armed bandit setting, with discrete arms, and high-dimensional covariates, generalizing results of \cite{goldenshluger2013linear}. \cite{bastani2016decision} present an algorithm, using a LASSO estimator, that obtains regret $O\left(K (\log T+\log d)^2\right)$ where $K$ denotes the number of arms. In contrast, our setting can be interpreted as a multi-armed bandit with continuous arms in a high dimensional space.

\end{itemize}


\item{{\bf High Dimensional Statistics.}} There has been a great deal of work on regularized estimator under the high-dimensional scaling; see e.g.~\cite{van2008high}. 
%
%
%
%
Closer to the spirit of our work is the problem of 1-bit compressed sensing~\cite{plan2013one,bhaskar20151}. 
In this problem, linear measurements are observed for an unknown parameter of interest but only the sign of these measurements are observed. Note that in our problem, seller is involved in both the learning task and also the policy design. Specifically, he should decide on the prices, which directly affect collected revenue and also indirectly influence the difficulty of the learning task. The market values are then compared with the posted prices, in contrast to 1-bit compressed sensing where the measurements are compared with zero (sign information).
In addition, the pricing problem has an online nature while the 1-bit compressed sensing is mostly studied for offline setting. Finally, note that 
prices are set based on customer's purchase behavior, and hence introduce dependency among the collected information about the model parameters.   
\end{description}


\subsection{Notations}
{For a vector $v$, $\supp(v)$ represents
the positions of nonzero entries of $v$.
Further, for a vector $v$ and a subset $J$, 
$v_J$ is the restriction of $v$ to indices in $J$.
We write $\|v\|_p$ for the standard $\ell_p$ norm of a vector $v$, i.e., $\|v\|_p = (\sum_i |v_i|^p)^{1/p}$
and $\|v\|_0$ for the umber of nonzero entries of $v$.  If the subscript $p$ is omitted, it should be deemed as $\ell_2$ norm.
For two vectors $a, b\in \reals^d$, the notation $a\cdot b = \sum_{i=1}^d a_i b_i$ represents the standard inner product.
For two
functions $f(n)$ and $g(n)$, the notation $f(n) =O( g(n) )$ means that $f$ is bounded above by $g$ asymptotically,
namely, $f(n) \le C g(n)$ for some fixed positive constant $C>0$. 
Throughout, $\phi(x) = e^{-x^2/2}/\sqrt{2 \pi}$ is the Gaussian density and $\Phi(x) \equiv\int_{-\infty}^x
\phi(u) \de u$ is the Gaussian distribution.}

%
%


\section{Choice model} \label{sec:model}
We consider a seller, who has a product for sale in each period $t=1,2,\cdots,T$, where $T$ denotes the length of the horizon and may be unknown the to the seller.
Each product is represented by an {\em observable} vector of features (covariates) $x_t \in \cX \subseteq \reals^d$. Products may vary across periods and we assume that feature vectors $x_t$ are sampled independently from a fixed, but a priori {\em unknown}, distribution $\pX$, supported on a bounded set $\cX$.

The product at time $t$ has a market value $v_t = v(x_t)$, which is {\em not observed} by the seller and function $v$ is (a priori) unknown.
At each period $t$, the seller posts a price $p_t$. If $p_t\le v_t$, a sale occurs, and the seller collects revenue $p_t$. If the price is set higher than the market value, $p_t>v_t$, no sale occurs and no revenue is obtained. 
The goal of the seller is to design a pricing policy that maximizes the collected revenue.

We first assume that the market value of a product is a linear function of its covariates, namely 
\begin{align}\label{eq:model}
v(x_t) = \tth\cdot x_t+\alpha_0 + z_t\,,
\end{align}
where $a\cdot b$ denotes the inner product of vectors $a$ and $b$. Here, $\{z_t\}_{t\ge 1}$ are idiosyncratic  shocks, referred to as noise, which are drawn independently and identically from a distribution with mean zero and cumulative function $F$, with density $f(x) = F'(x)$,~cf.~\cite{keskin2014dynamic}.
The noise can account for the features that are not measured. 
We generalize our model to non-linear valuation functions in Section~\ref{sec:general}.

Parameter $\tth$ is a prior unknown to seller. Therefore, the seller is involved in the realm of exploration-exploitation as he needs to choose between learning $\tth$ and exploiting what has been learned so far to collect revenue. 

{Henceforth, we let $\p_0 = (\tth,\alpha_0) \in \reals^{d+1}$ denote the true model parameters and also define the augmented feature vectors $\tx_t = (x_t,1)$.}

Let $y_t$ be the response variable that indicates whether a sale has occurred at period $t$:
\begin{align}\label{mymodel}
y_t = \begin{cases}
+1&\text{ if }v_t \ge p_t\,,\\
-1&\text{ if }v_t <p_t\,.
\end{cases}
\end{align}
Note that the above model can be represented as the following probabilistic model:
\begin{align}\label{eq:prob-model}
y_t = \begin{cases}
+1&\text{ with probability }\, 1- F\left(p_t-\p_0\cdot \tx_t\right)\,,\\
-1&\text{ with probability }\, F\left(p_t- \p_0\cdot \tx_t\right)
\end{cases}
\end{align}

Our proposed algorithm exploits the structure (sparsity) of the feature space to improve its performance. 
To this aim, let $s_0$ denote the number of nonzero coordinates of $\tth$, i.e., $s_0  =  \|\p_0\|_0 = \sum_{j=1}^d \ind(\p_{0j}\neq 0)$. We remark  that $s_0$ is a priori unknown to the seller.


\subsection{Technical assumptions}
To simplify the presentation, we assume that $\|x_t\|_\infty \le 1$, for all $x_t\in \cX$, and  $\|\p_0\|_1\le \l1u$ for a known constant $\l1u$, where for a vector $u = (u_1,\dotsc, u_d)$, $\|u\|_\infty= \max_{i\in[d]}|u_i|$ denotes the maximum absolute value of its entries and $\|u\|_1 = 
\sum_{i=1}^d |u_i|$.
We denote by $\Omega$ the set of feasible parameters, i.e., 
$$\Omega = \Big\{\p \in \reals^{d+1}:  \|\p\|_0 \le s_0\,, \,\, \|\p\|_1\le \l1u\Big\}\,.$$

We also make the following assumption on the distribution of noise $F$.
\begin{assumption}\label{ass1}
The function $F(v)$ is strictly increasing. Further, $F(v)$ and $1-F(v)$ are log-concave in $v$.
\end{assumption}

Log-concavity is a widely-used assumption in the economics literature~\cite{bagnoli2005log}.
Note that if the density $f$ is symmetric and the distribution $F$ is log-concave, then $1-F$ is also log-concave.
Assumption~\ref{ass1} is satisfied by several common probability distributions including normal, uniform, Laplace, exponential, and logistic.
Note that the cumulative distribution function of all log-concave densities is also log-concave~\cite{boyd2004convex}.

Our second assumption is on the product feature vectors.

\begin{assumption}\label{ass2}
{Product feature vectors are generated independently from a probability distribution $\pX$ with a bounded support $\cX\in \reals^d$}. We further assume that $\E(x_t)$ is normalized to zero\footnote{This normalization does not imply any restriction because if $\E(x_t)\neq 0$, then it can be absorbed in the intercept term $\alpha_0$. More precisely, we consider model with intercept parameter $\tilde{\alpha_0} = \alpha_0 + \theta_0\cdot \E(x_t)$.} and denoting by $\Sigma = \E(x_tx_t^\sT)$ the covariance matrix of $\{x_t\}$, we assume that $\Sigma$ is a positive definite matrix. Namely, all of its singular values are bounded from below by a constant $C_{\min}> 0$. We also denote the maximum eigenvalue of $\Sigma$ by $C_{\max}$.
\end{assumption}

The above assumption holds for many common probability distributions, such as uniform, truncated normal, and in general truncated version of many more distributions. Generally, if $\pX$ is bounded below from zero on an open set around the origin, then it has a positive definite covariance matrix. 
{Let us stress that we know neither the distribution $\pX$, nor its covariance $\Sigma$.}

\subsection{Clairvoyant policy and performance metric}\label{sec:Benchmark}
We evaluate the performance of our algorithm using the common notion of regret: the expected revenue loss compared with the optimal pricing policy that knows $\p_0$ in advance (but not the realizations of $\{z_t\}_{t\ge 1}$). Let us first characterize this benchmark policy.

Using Eq.~\eqref{eq:model}, the expected revenue from a posted price $p$ is equal to 
$$p\times\prob(v_t\ge p)= p(1-F(p-\p_0\cdot \tx_t ))\,.$$ 
Therefore, using first order conditions, for the optimal posted price, denoted by $p^*$, we have
\begin{align}\label{eq:opt-p}
\popt(x_t) = \frac{1-F(\popt-\p_0\cdot \tx_t)}{f(\popt-\p_0\cdot \tx_t)}\,.
\end{align}
To simplify the presentation, let $p^*_t=p^*(\tx_t)$ denote the optimal price at time $t$.

We now define $\varphi(v)\equiv v - \frac{1-F(v)}{f(v)}$ corresponding to the \emph{virtual valuation} function commonly used in mechanism design~\cite{Myerson81}.
By Assumption~\ref{ass1}, $\varphi$ is injective and hence we can define function $g$ as follows
%
%
\begin{align}\label{eq:g}
g(v) \equiv v + \varphi^{-1}(-v)\,.
\end{align}
It is easy to verify that $g$ is non-negative. Note that by Eq.~\eqref{eq:opt-p}, for the optimal price we have $$\p_0\cdot \tx_t + \varphi(\popt-\p_0\cdot \tx_t)=0.$$ 
Therefore, by rearranging the terms for the optimal price at time $t$ we have
\begin{align}\label{eq:popt}
\popt_t = g(\p_0\cdot \tx_t)\,.
\end{align}


We can now formally define the regret of a policy. Let $\pi$ be the seller's policy that sets price $p_t$ at period $t$, and $p_t$ can depend on the history of events up to time $t$. The worst-case regret is defined as:
\begin{align}\label{eq:Regret_def}
\Reg_\pi(T) \equiv \max_{\substack{\p_0 \in \Omega\\ \pX\in Q(\cX)}} \E \left[\sum_{t=1}^T \bigg(p^*_t \ind(v_t \ge p^*_t) - p_t \ind(v_t \ge p_t) \bigg)\right]\,,
\end{align}
where the expectation is with respect to the distributions of idiosyncratic noise, $z_t$, and $\pX$, the distribution of feature vectors. Moreover, $Q(\cX)$ represents the set of probability distributions
supported on a bounded set $\cX$.

Our algorithm uses the sparsity structure of $\p_0$ and learns the model with order of magnitude less data compared to a 
structure-ignorant algorithm. In Section~\ref{sec:regret}, we show that our pricing scheme achieves a regret bound of 
$O\big(s_0 \log T (\log d + \log T)\big)$.

%
%
%
%
%
%

\section{A Regularized Maximum Likelihood Pricing (RMLP) Policy} \label{sec:pricing_alg}

\begin{algorithm}[t]
\begin{algorithmic}[1]

\REQUIRE{\bf (at time $0$)} function $g$, regularizations $\lambda_k$, $\l1u$ (bound on $\|\p_0\|_1$),\hspace{2cm}
\REQUIRE{\bf (arrives over time)} covariate vectors $\{\tx_t\}_{t\in \naturals}$ 

\ENSURE prices $\{p_t\}_{t\in \naturals}$ 

\STATE $\tau_1 \leftarrow 1$, $p_1 \leftarrow 0$, $\hp^1 \leftarrow 0$

\FOR{each episode $k = 2,3,\dots$}
\STATE Set the length of $k$-th episode: $\tau_k \leftarrow 2^{k-1}.$

\STATE Update the model parameter estimate $\hp^{k}$ using the regularized ML estimator obtained\\ from observations in the previous episode:
\begin{align}\label{eq:ML}
\hspace{-4cm}  \hp^{k}  = \underset{\|\p\|_1 \le \l1u}{\text{arg min}} \,\, \left\{ \cL(\p) + \lambda_k \|\p\|_1 \right\} \hspace{-2cm}
\end{align}
with 
\begin{align} \label{eq:log_likelihood}
\hspace{-8.6cm} \cL(\p) = - \frac{1}{\tau_{k-1}}\sum_{t=\tau_{k-1}}^{\tau_{k}-1} \bigg\{\ind(y_t =1) \log (1-F(p_t - \p\cdot \tx_t )) + \ind(y_t =-1) \log (F(p_t - \p\cdot \tx_t )) \bigg\} 
\hspace{-6.8cm}
\end{align}
%
%

\STATE For each period $t$ during the $k$-th episode, set
\begin{align}\label{eq:price}
\hspace{-2cm} p_t \leftarrow g( \hp^k\cdot \tx_t) \hspace{-1.2cm}
\end{align}

\ENDFOR
\end{algorithmic}
\caption{\bf RMLP policy for dynamic pricing}\label{alg-linear}
\end{algorithm}


In this section, we present our dynamic pricing policy. 
Our policy runs in an episodic fashion. 
Episodes are indexed by $k$ and time periods are indexed by $t$. 
The length of episode $k$ is denoted by $\tau_k$. 
Throughout episode $k$, we set the prices equal to $p_t = g(\tx_t\cdot\hp^k)$ where $\hp^{k}$ denotes the estimate of $\p_0$ which is obtained from the observations $\{(x_t,y_t,p_t)\}$ in the \emph{previous} episode. Note that by Eq.~\eqref{eq:g}, $p_t$ is the optimal posted price if $\hp^k$ was the true underlying parameter of the model. 

We estimate $\p_0$ using a regularized maximum-likelihood estimator; see Eq.~\eqref{eq:ML} where the (normalized) negative log-likelihood function for $\p$ is given by Eq.~\eqref{eq:log_likelihood}.
We note that as a consequence of the log concavity assumption on $F$ and $1-F$, the optimization problem \eqref{eq:ML} is a convex problem. {There is a large toolkit of various optimization methods (e.g., alternating direction method of multipliers (ADMM), fast iterative shrinkage-thresholding algorithm (FISTA), accelerated projected gradient descent, among many others) that can be used to solve this optimization problem. There are also recent developments on distributed solvers for $\ell_1$ regularized cost function~\cite{Boyd2011}. }

Observe that by design, prices posted in the $k$-th episode are independent from the market value noises in this period, i.e., $\{z_t\}_{t=\tau_k}^{\tau_{k+1}-1}$.  This allows us to estimate $\p_0$ for each episode separately; see Proposition~\ref{thm:learning} in Section~\ref{proof:thm2}. {Comparing to policies that use the entire data sale history in making decisions, some remarks are in order:}
{
\begin{itemize}
\item \emph{Perishability of data:} In practical applications, the unknown demand parameters will change over time, raising the concern of perishability of data. Namely, collected data becomes obsolete after a while  and cannot be relied on for estimating the model parameters~\cite{keskin2016chasing,Javanmard17}. Common practical policies to mitigate this problem (discussed in~\cite{keskin2016chasing}) include moving windows and decaying weights which use only recent  data to learn the model parameters. In contrast, methods that use the entire historical data suffers from this problem.
\item \emph{Simplicity and efficiency:} In RMLP policy, estimates of the model parameters are updated only at the first period of each episode ($\log T$ updates). Further, at each update, the policy uses only the historical data from the previous episode. These two ideas together, not only allow for a neat analysis of the statistical dependency among samples but also decrease the computational cost. Scalability of the pricing policy is indispensable in practical applications as the sales data is collected at an unprecedented rate.    

\item \emph{Effect on regret:} By using half of the historical data at each update, our policy loses at most a factor $2$ in the total regret. (This becomes clear shortly when we discuss the estimation error rate in terms of number of samples.) 
\end{itemize}
}

The lengths of episodes in our algorithm increase geometrically ($\tau_k=2^{k-1}$), allowing for more accurate estimate of $\p_0$ as the episode index grows. The algorithm terminates at the end of the horizon (period $T$), but note that it does not need to know the length of the horizon in advance.

Regularization parameter $\lambda_k$ constrains the $\ell_1$ norm of  the estimator $\hp^k$. Selecting the value of $\lambda_k$ is of crucial importance as it effects the  estimator error.
We set it as $\lambda_k = O\left(\sqrt{{(\log d)}/{\tau_{k-1}}}\right)$. 
More precisely, define
\begin{eqnarray} 
\nonumber u_{\l1u} &\equiv& \sup_{|x|\le 3\l1u} \left\{\max\Big\{  \log' F(x) , -\log'(1-F(x)) \Big\}\right\}\,,\label{eq:uM}
\end{eqnarray}
where the derivatives are w.r.t. $x$.
By the log-concavity property of $F$ and $1-F$, we have 
$$u_{\l1u} = \max\Big\{  \log' F(-2\l1u) , -\log'(1-F(2\l1u)) \Big\}\,.$$ Hence, $u_{2\l1u}$ captures the steepness of $\log F$. 



In order to minimize the regret, we run the RMLP policy with 
\begin{equation} \label{eq:lambda}
\lambda_k = 4u_{\l1u} \sqrt{\dfrac{\log d}{\tau_{k-1}}}.
\end{equation} 

Note that exploration and exploitation tasks are mixed in our algorithm.
In the beginning of each episode, we use what is learned from previous episode to improve the estimation of  $\tth$
and then we exploit this estimate throughout the current episode to incur little regret. Meanwhile, the observations
gathered in the current episode are used to update our estimate of $\tth$ for the next episode.
We analyze the performance of RMLP in the next section.

\section{Regret analysis}\label{sec:regret}
Although the description of RMLP is oblivious to sparsity $s_0$, its performance depends on the structure of the optimal solution.
The following theorem bounds the regret of our dynamics pricing policy.
\begin{theorem}[Regret Upper Bound]\label{thm:regret}
Suppose Assumptions~\ref{ass1} and \ref{ass2} hold. 
Then, the regret of the RMLP policy is of $O\big(s_0 \log d \cdot \log T\big)$.
\end{theorem}

Below we provide an outline for the proof of Theorem~\ref{thm:regret} and defer its complete proof to Section~\ref{proof:thm2}.

\begin{enumerate}
\item {In RMLP, the updates in the model parameter estimation only occurs at the beginning of each episode, with using only the samples collected in the previous episode. Therefore,
the prices posted in each episode are independent from the market value noises in that episode. This observation also verifies that $\cL(\p)$ given by~\eqref{eq:log_likelihood}, is indeed the negative log-likelihood  of the samples collected in $k$-th episode. Note that this independence is not a mere serendipity, rather it holds because of the specific design of RMLP policy. Using this property, we use tools from high-dimensional statistics to bound the estimation error. To bound the error term $\|\p^k -\p_0\|_2$, we compare the function values $\cL(\p^k)$ and $\cL(\p_0)$. The main challenge here is that $\cL(\p)$ is not strictly convex in $\p$.\footnote{Note that $\nabla^2_{\theta}\cL = (-1/{\tau_{k-1}}) \sum_{t=\tau{k-1}}^{\tau_k -1} (\partial^2/\partial^2_{u_t} \cL) x_tx_t^\sT $, where $u_t = p_t-\theta\cdot x_t - \alpha_0$. Therefore, $\nabla^2_{\theta}\cL$ is a $d\times d$ matrix of rank at most $\tau_k -\tau_{k-1}$. Hence, $\cL(\p)$ is strictly convex in $\p$ only if $\tau_k- \tau_{k-1} \ge d$.  However, since we are not updating our estimates in the middle of an episode, episodes of length $d$ yield the regret to scale linearly in $d$, which is not desired.} Hence, there can be, in principle, parameter vectors $\p_1$ and $\p_2$ that are close to each other and nevertheless the values of function $\cL$ at these points are far from each other.} 

{To cope with this challenge, we show that a so-called \emph{restricted eigenvalue condition} holds for the feature products. This notion implies that $\cL(\p)$ is strictly convex on the set of sparse vectors.\footnote{It is strictly convex over the set of $s_0$ sparse vectors in $d$-dimension if the number of samples is above $c s_0 \log d$ for a suitable constant $c>0$.}}   
Using the restricted eigenvalue condition, we show the following $\ell_2$ error for the regularized log-likelihood estimate in the $k$-th episode, $\hp^k$, holds true
$$\|\hp^k-\p_0\|_2 = O(\sqrt{s_0}\lambda_k) = O\left(\sqrt{\frac{s_0\log d}{\tau_{k-1}}}\right)\,.$$ 
As expected, the estimate gets more accurate
as the episode's length increases; see Section~\ref{proof:thm2} for more details.

\item For any $p\ge 0$, denote by $r_t(p) = p(1-F(p-\tx_t\cdot\p_0))$, the expected revenue under price $p$. We bound $R_t$ in terms of $r_t(p^*_t) - r_t(p_t)$. Since $p^*_t\in \arg\max \{r_t(p)\}$, we have $r_t'(p^*_t) = 0$, and by Taylor expansion of $r_t$ around $p^*_t$, we obtain $r_t(p^*_t ) - r_t(p_t) = O((p^*_t-p_t)^2)$.
\item For $t$ in the $k$-th episode, namely $\tau_{k-1}\le t\le \tau_k-1$, we have 
$$p^*_t - p_t = g(\p_0 \cdot \tx_t ) -g(\hp^k\cdot \tx_t ) \le |(\p_0- \hp^k) \cdot \tx_t |\,,$$ 
which follows by showing that $g$ is $1$-Lipschitz.  Further, by Assumption~\ref{ass2} (without loss of generality assume $C_{\max} > 1$), we have 
\begin{align*}
\E[((\p_0- \hp^k) \cdot \tx_t )^2]  
\le C_{\max}\E[\|\hp^k- \p_0\|_2^2] \,,
\end{align*}
where the equality holds because $x_t$ is independent of $\hp^k$. The inequality holds because $\E(x_t) = 0$ and therefore 
\begin{align}\label{eq:sigma-aug}
\E(\tx_t \tx_t^\sT) =  \begin{bmatrix}
\Sigma & 0\\
0& 1
\end{bmatrix}\,,
\end{align}
from which we obtain that the maximum eigenvalue of $\E(\tx_t \tx_t^\sT)$ is at most $C_{\max}>1$.

Let $R_t$ be the regret occurred at step $t$. Combining the above bounds (step 2 and 3), we arrive at $\E[R_t] = O({s_0(\log d)}/{\tau_{k-1}})$. Therefore, the cumulative expected regret in episode $k$ works out 
at $O(s_0\log  d)$.  Since the length of episodes increase geometrically, there are $O(\log T)$ episodes by time $T$. This implies that the total expected regret by time $T$ is  $O(s_0 \log d \, \log T )$.
\end{enumerate}


\subsection{Comparison with the ``common" regret of bound $\Omega(\sqrt{T})$}\label{uninformative}
There is an often-seen regret bound $\Omega(\sqrt{T})$ in the literature of online decision making, which can be improved to a logarithmic regret bound if some type of ``separability assumption" holds true~\cite{dani2008stochastic,abbasi2012online}. Separability assumption posits that there is a positive constant gap between the rewards of the best and the second best actions. In our framework, the parameter $\p$ belongs to a continuous set in $\reals^{d+1}$ and therefore the separability assumption cannot be enforced as by choosing $\p$ arbitrary close to $\p_0$, one can obtain suboptimal (but arbitrary close to optimal) reward.  However, our policy achieves $O(\log T)$ regret. Here, we contrast our logarithmic lower bound with the folklore bound $\Omega(\sqrt{T})$ to build further insight on our results.

\paragraph{\bf Uninformative prices and $\Omega(\sqrt{T})$ lower-bound.} We focus on~\cite{broder2012dynamic} which has a close framework to ours in that it considers a dynamic pricing policy from purchasing decisions and presents a pricing policy based on maximum likelihood estimation with regret $O(\sqrt{T})$. Adopting their notation, it is assumed that market values $v_t$ are independent and identically distributed random variables coming from a distribution function that belongs to some family parametrized by $\bz$. Denote by $d(p;\bz)$ the demand curve. This curve determines the probability of a purchase at a given price, i.e., $d(p;\bz) = \prob_\bz(v_t\ge p)$. 
\cite{broder2012dynamic} show that the worst-case regret of any pricing policy must be at least $\Omega(\sqrt{T})$ (see Theorem~3.1 therein). The bound is proved by considering a specific family of demand curves $d(p;\bz)$, such that all demand curves in this family intersect at a common price. Further, the common price is the optimal price for a specific choice of parameter $\bz_0$, i.e, $p^*(\bz_0)$.\footnote{Specifically, they consider $d(p;\bz) = 0.5+\bz-\bz p$. Hence $d(1;\bz) = 1$, for all $\bz$ and it is shown that $p^*(\bz_0) = 1$ for $\bz_0 = 0.5$.}   Therefore, the price $p^*(\bz_0)$ is ``uninformative" since no policy can gain information about the demand parameter $\bz$, while pricing $p^*(\bz_0)$. The idea behind the derived lower bound for the worser-case regret is that for a policy to learn the underlying demand curve fast enough, it must necessarily choose prices that are away from (the uninformative) price $p^*(\bz_0)$ and this leads to a large regret when the true demand curve is indeed $\bz_0$.

{
\paragraph{\bf Intuition behind our results.} In contrast to the previous case, for our framework there is no such uninformative price. First, note that the for a choice model with parameters $\p_0 = (\tth, \alpha_0)$,  the demand curve at time $t$ is given by
$$d_t(p;\p_0) = 1-F(p-\p_0\cdot \tx_t)\,,$$ 
For $n\ge 1$, we define the aggregate demand function up to time $n$ as $d_1^n = (d_1, d_2, \dotsc, d_n)$.
In the following, we argue that under our setting, there is no uninformative price.
For any price $p$ and any $\p_1$, $\p_2$, we have 
\begin{align*}
\frac{1}{n}\| d_1^n(p,\p_1) - d_1^n(p,\p_2) \|_2^2  
&= \frac{c^2}{n} \sum_{\ell=1}^n ((\p_1-\p_2)\cdot \tx_t)^2\\
& = \frac{c^2}{n} \|\widetilde{X}(\p_1-\p_2)\|_2^2\,,
\end{align*}
where $\widetilde{X}$ is the matrix with rows $\tx_\ell$, for $1\le \ell\le t$. 
We also used the fact that $f(z)\ge c>0$ for some constant $c$ because $F$ is strictly increasing by Assumption~\ref{ass1}. 
As we show in Appendix~\ref{app:RE}, for $n\ge c_0 s_0 \log d$ (with $c_0$ a proper constant), $\widetilde{X}$ satisfy a so-called ``restricted eigenvalue", by which we have 
\begin{align}
\frac{1}{n}\|\widetilde{X}(\p_1-\p_2)\|_2^2\ge \frac{C_{\min}}{2} \|\p_1-\p_2\|_2^2\,.
\end{align}
Therefore, for any fixed price $p$, if we vary the demand parameters $\mu_1$ to some other value $\widehat{\mu}_1$, then the aggregate demand at price $p$ also changes by an amount proportional to $\|\p_1-\p_2\|_2$. Hence, any price in this setting is informative about the model parameters.} 

{ To build further insight, let us consider a more general choice model, where the utility of the customer from buying a product with feature vectors $x_t$ at price $p$ is given by
\begin{align}\label{eq:utility}
u(x_t) = \theta_0 \cdot x_t +\alpha_0 - \beta_0 p +z_t\,,
\end{align} 
where $\theta_0, \alpha_0, \beta_0$ are unknown model parameters and $z_t$ is the noise term. The customer buys the product iff $u(x_t) \ge 0$. Note that the model we studied in this paper (see Equation~\eqref{mymodel}) is special case when the price sensitivity $\beta_0$ is known and hence can be normalized to $1$. We next argue that in case of unknown $\beta_0$, the uninformative prices do exist and hence the $\Omega(\sqrt{T})$ is still in place.}

{
To see this, fix  arbitrary $\alpha_*$, and let $\theta_0 = 0$ and $\beta_0 = g(\alpha_*)-\alpha_*+ \alpha_0$. Then, the demand curves will be unaltered over time and are given by
$$d_t(p,\p) = 1-F(\beta_0 p -\alpha_0) = 1-F\left((g(\alpha_*)-\alpha_*+ \alpha_0) p - \alpha_0\right)\,.$$  
It is easy to verify that $p^* = 1$ is the optimal price for the specific choice of $\alpha_0 = \alpha_*$. Further, all the demand curves intersect at $p^* = 1$ (they all have the value $1-F(g(\alpha_*)-\alpha_*)$ at this price).
Therefore, $p^*$ is an uninformative price and no policy can gain information about $\alpha_0$ by pricing at $p^*$.  However, when $\alpha_0 = \alpha_*$, choosing prices that are away from this informative price leads to a large regret. Prices that are close to $p^*$ does not have any information gain, and contrasting these two points, it can be shown that the worst  case regret id of order $\Omega(\sqrt{T})$. A formal proof follows the same lines ad the proof of~\cite[Theorem 3.1]{broder2012dynamic} and is omitted. } 

{
Finally, it is worth noting that the rate of learning demand parameter $\p_0$ is chiefly derived by three factors:
\begin{itemize}
\item Non-smoothness of distribution function $F$, as it controls the amount of information obtained about $\tx_t\cdot \p_0$  at each $t$. This is captured by quantity $\ell_{\l1u}$ defined by~\eqref{eq:lM}.
\item The rate by which the feature vectors $x_t$ span the parameter space. This is controlled through the minimum eigenvalue of $\Sigma$, i.e., $C_{\min}$. If $C_{\min}$ is small, the randomly generated features are relatively aligned and one requires larger sample size to estimate $\tth$ within specified accuracy.
\item Complexity of $\p_0$. This is captured through the sparsity measure $s_0$.  
\end{itemize}
Contribution of these factors to the learning rate can be clearly seen in our derived learning bound~\eqref{eq:B3}.     
}


\subsection{Role of $C_{\min}$  }

In establishing our results, we relied on Assumption~\ref{ass2} which requires the population covariance of features to be positive definite.
The lower bound on its eigenvalues, denoted by $C_{\min}$, appears in our regret bound as a factor $1/C_{\min}^2$. 

As evident from the proof of Proposition~\ref{thm:learning}, Assumption~\ref{ass1} can be replaced by the weaker  {\em restricted eigenvalue condition}~\cite{buhlmann2011statistics,candes2007dantzig}, which is a common assumption in high-dimensional statistical learning. 
While assumption $C_{\min} > 0$ allows for a fast learning rate of model parameters and a regret bound $O(\log T)$, RMLP policy can still provably achieve regret $O(\sqrt{T})$, even when $C_{\min} = 0$.

\begin{theorem}\label{thm:NoC}
Suppose that product feature vectors are generated independently from a probability distribution $\pX$ with a bounded support $\cX\in \reals^d$. Under Assumption~\ref{ass1}, the regret of RMLP policy is of $O(\sqrt{(\log d) T})$. 
\end{theorem}

Proof of Theorem~\ref{thm:NoC} is given in Section~\ref{proof:thm-NoC}. 

\section{Lower bound on regret} \label{sec:lower-bound}
As discussed in Section~\ref{sec:Benchmark}, if the true parameter $\p_0$ is known, the optimal policy (in terms of expected revenue) is the one that sets prices as $p_t = g(\tx_t \cdot \p_0)$. Let $\cH_{t} = \{x_1,x_2,\dotsc, x_t, z_1, z_2, \dotsc, z_t\}$ denote the history set up to time $t$, and recall that $\Omega$ denotes the set of feasible parameters, i.e., $\Omega = \{\p\in \reals^{d+1}:\|\p\|_0\le s_0\,, \,\,\|\p\|_1\le W\}$.
We consider the following set of policies, $\Pi$:
\begin{align}\label{eq:Pi}
\Pi = \Big\{\pi: \pi(p_t) = g(\tx_t\cdot \p_t), \text{ for some } \p_t\in \Omega, \text{ such that }\p_t \text{ is } \cH_{t-1}\text{-measurable} \Big\}\,.
\end{align}
Here $\pi(p_t)$ denotes the price posted by policy $\pi$ at time $t$.

We provide a lower bound on the achievable regret by any policy in set $\Pi$. Indeed this lower bound applies to an oracle who fully observes the market values after the price is either accepted or rejected. Compared to our setting, where the seller observes only the binary feedbacks (purchase/no purchase), this oracle appears exceedingly powerful at first sight but surprisingly, the derived lower bound matches the regret of our dynamic policy, up to a logarithmic factor.    

\begin{theorem}\label{thm:LB-regret}
Consider linear model~\eqref{eq:model} with $\alpha_0 = 0$, where the market values $v(x_t)$, $1\le t\le T$, are fully observed. We further assume that market value noises are generated as $z_t\sim \normal(0,\sigma^2)$.
Let $\Pi$ be the set of policies given by~\eqref{eq:Pi}. Then, there exists constant $C'>0$ (depending on $\l1u$ and $\sigma$), such that the following holds true for all $T\in \naturals$.
\begin{align}
\min_{\pi\in \Pi}\Reg_\pi(T) \ge C' \bigg\{ s_0 \log \Big(\frac{T}{s_0}\Big)+ \min\bigg[\frac{T}{s_0},  s_0 \log \Big(\frac{d}{s_0}\Big)\bigg]\bigg\} \,.
\end{align} 
\end{theorem}

\medskip
In the following we give an outline for the proof of Theorem~\ref{thm:LB-regret}, summarizing its main steps and defer the complete proof to Section~\ref{proof:LB-regret}.
\begin{enumerate}
\item We derive a lower bound for regret in terms of the minimax estimation error. 
Specifically, for $t\in \naturals$, let
\begin{align}
R_t\equiv p^*_t \ind(v_t \ge p^*_t) - p_t \ind(v_t \ge p_t)
\end{align}
be the regret at period $t$. Define $\Omega_0 = \{\theta\in\reals^d:\, (\theta,0)\in \Omega\}$. We show that 
\begin{align}
\max_{\tth\in \Omega_0} \E(R_t)\ge c \max_{\tth\in \Omega_0} \E\{\min(\|\th_t-\tth\|_2^2,C)\}\,,
\end{align} for some constants $c,C>0$. 
 \item Let $\th_1^T=(\th_t)_{t=1}^T$ and define $\dis(\th_1^T,\th)\equiv \sum_{t=1}^T \min(\|\th_t - \th\|^2_2,C)$. We use a standard argument (Le Cam's method) that relates the minimax $\ell_2$ risk,
 $\min_{\th_1^T} \max_{\tth\in \Omega_0} \E\dis(\th_1^T,\tth)$,  in terms of the error in multi-way hypothesis problem~\cite{tsybakov2008introduction}. 
 We first construct a maximal set of points in $\Omega_0$, such that minimum pairwise distances among them is at least $\delta$. (Such set is usually referred to as a $\delta$-packing in the literature). Here $\delta$ is a free parameter to be determined in the proof. We then use a standard reduction to show that any estimator with small minimax risk should necessarily solve a hypothesis testing problem over the packing set, with small error probability. More specifically, suppose that 
 nature chooses one point from the packing set uniformly at random and conditional on nature's choice of the parameter vector, say $\tth$, the market value are generated according to $\<x_t,\tth\>+z_t$ with $z_t\sim\normal(0,\sigma^2)$. The problem is reduced to lower bounding the error probability in distinguishing $\tth$ among the candidates in the packing set using the observed market values.
 \item We apply Fano's inequality from information theory to lower bound the probability of error~\cite{tsybakov2008introduction}. The Fano bound involves the logarithm of the cardinality of the $\delta$-packing set as well as the mutual information between the observations (market values) and the random parameter vector $\tth$ chosen uniformly at random from the packing set. {Le Cam's method is used to derive minimal risk lower bound
 for an estimator $\hth$, while here we have a sequence of estimators and need to adjust the Le Cam's method to get the lower bound for $\dis(\th_1^T,\th_0)$.}  
 
 \end{enumerate}
 
\section{Nonlinear valuation function} \label{sec:general}

In previous sections, we focused exclusively on linear valuation function given by Eq~\eqref{eq:model}.
Here, we extend our results and assume that the market valuations are modeled by a \emph{nonlinear} function that depends on products' features and an independent noise term. Specifically, the market value of a product with feature vector $x_t$ is given by
\begin{align}\label{eq:nonlinear_model}
v(x_t) = \f(\theta_0 \cdot \phi(x_t) + \alpha_0 +z_t)\,,
\end{align} 
where the original features $x_t$ are transformed by a feature mapping $\phi:\reals^d \mapsto \reals^d$, and function $\psi:\reals \mapsto \reals$ is a general function that is log-concave and strictly increasing.
Important examples of this model include log-log model ($\psi(x) = e^x$, $\phi(x) = \ln(x)$), semi-log model ($\psi(x) = e^x$, $\phi(x) = x$), and logistic model ($\psi(x) = e^x/(1+e^x)$, $\psi(x) = x$).

Model~\eqref{eq:nonlinear_model} allows us to capture correlations and non-linear dependencies on the features. We next state our assumption on the feature mapping $\phi$ and then discuss our dynamic pricing policy and its regret bound for the general setting~\eqref{eq:nonlinear_model}.  
 


\begin{assumption}\label{ass2-2} 
Let $p_X$ be an (unknown) distribution
from which the original features $x_t$ are sampled independently.
Suppose that the feature mapping $\phi$ has continuous derivative and
denote by $\Sigma_\phi \equiv \E(\phi(x)\cdot \phi(x)^\sT)$, the covariance of feature vector $\phi(x)$ under $\prob_X$. 
We assume that there exist constants $C_{\min}$ and $C_{\max}$ such that for every eigenvalue $\sigma$ of $\Sigma_\phi$, we have $0<C_{\min}\le \sigma <C_{\max}<\infty$.
\end{assumption}

Invoking Assumption~\ref{ass1}, $\prob_X$ has a bounded support $\cX$ and since $\phi$ has continuous derivative, it is Lipschitz on $\cX$ and hence the image of $\cX$ under $\phi$ remains bounded.
Therefore, the new features $\phi(x_t)$ are also sampled independently from a bounded set. 
The condition on $\Sigma_\phi$ is analogous to that on $\Sigma$, as required by Assumption~\ref{ass2} for the linear setting.  

Based on feature mapping $\phi$, validity of Assumption~\ref{ass2-2} may depend on all moments of distribution $\prob_X$. We provide an alternative to this assumption, which only depends on feature mapping $\phi$ and the second moment of $\prob_X$. In stating the assumption, we use the notation $\cD_\phi$ to denote the derivative matrix of a feature mapping $\phi$. Precisely, for $\phi = (\phi_1,\dotsc, \phi_d)$, with $\phi_i$ real-valued function defined on $\reals^d$, we write $\cD_\phi = ({\partial \phi_i}/{\partial x_j})_{1\le i\le j\le d}$. 

\begin{assumption}\label{ass2-3}
Suppose that feature mapping $\phi$ has continuous derivative and its derivative $\cD_\phi(x)$ is full-rank for almost all $x$. In addition, there exist constants $C_{\min}$ and $C_{\max}$ such that for every eigenvalue $\sigma$ of covariance $\Sigma$, we have $0<C_{\min}\le \sigma <C_{\max}<\infty$.
\end{assumption}

Recall that the noise terms $\{z_t\}_{t\ge 1}$ are drawn independently and identically from a distribution with cumulative function $F$ and density $f(x)$. 
Let $\lambda(v) = f(v)/(1-F(v))$ be the hazard rate function for distribution $F$. For a log-concave function $\psi$, we define
\begin{align}\label{g:nonlinear}
g_\f^{-1}(v)\equiv v-\lambda^{-1} \Big(\frac{\psi'(v)}{\psi(v)}\Big)\,.
\end{align}
Note that $\psi'(v)/\psi(v) = \log'\psi(v)$ and since $\psi$ is log-concave, this term is decreasing. Further, since $1-F$ is log-concave then its hazard rate $\lambda$ is increasing (See proof of Lemma~\ref{app:phi-monotone}.) Combining these observations, we have that $-\lambda^{-1} (\psi'(v)/\psi(v))$ is increasing.  Consequently,
\begin{itemize}
\item[$\bullet$] Right-hand side of~\eqref{g:nonlinear} is strictly increasing and hence, $g_\f^{-1}$ is well-defined.
\item[$\bullet$] We have $(g_\f^{-1})'(v) \ge 1$, for all $v$. This implies that $0< g'_\psi(v)\le 1$, for all $v$.
\end{itemize}

It is worth noting that for $\f(v) =v$ (linear model), we have $g_\f = g$, where $g$ is defined by~\eqref{eq:g}. Our pricing policy for the nonlinear model is conceptually similar to 
the linear setting: The policy runs in an episodic manner. During episode $k$, the prices are set as $p_t = \psi(g_\psi(\hp^k \cdot \tx_t))$, where $\hp^k$ denotes the estimate of the true parameters $(\tth,\alpha_0)$
using a regularized maximum-likelihood estimator applied to observations in the \emph{previous} episode, and $\tx_t = (\phi(x_t),1)$.

We describe our (modified) RMLP policy in Algorithm~\ref{alg-nonlinear}. 
There a few differences between Algorithm~\ref{alg-nonlinear} and Algorithm~\ref{alg-linear}: Firstly, the features $x_t$ are replaced by $\tx_t = (\phi(x_t),1)$.
Secondly, in the regularized estimator, prices $p_t$ are replaced by $\f^{-1}(p_t)$. Thirdly, in the last step of algorithm prices are set 
as $\f(g_\f(\hp^k\cdot \tx_t))$, with $g_\f$ defined by Equation~\eqref{g:nonlinear}.

\begin{algorithm}[t]
\begin{algorithmic}[1]

\REQUIRE{\bf (at time $0$)} function $g$, regularizations $\lambda_k$, $\l1u$ (bound on $\|\tth\|_1$),\hspace{2cm}
\REQUIRE{\bf (arrives over time)} covariate vectors $\{\tx_t = (\phi(x_t),1)\}_{t\in \naturals}$ 

\ENSURE prices $\{p_t\}_{t\in \naturals}$ 

\STATE $\tau_1 \leftarrow 1$, $p_1 \leftarrow 0$, $\hp^1 \leftarrow 0$

\FOR{each episode $k = 2,3,\dots$}
\STATE Set the length of $k$-th episode: $\tau_k \leftarrow 2^{k-1}.$

\STATE Update the model parameter estimate $\hp^{k}$ using the regularized ML estimator obtained\\ from observations in the previous episode:
\begin{align}\label{eq:ML3}
\hspace{-3.8cm}  \hp^{k}  = \underset{\|\p\|_1 \le \l1u}{\text{arg min}} \,\, \left\{ \cL(\p) + \lambda_k \|\p\|_1 \right\} \hspace{-2.4cm}
\end{align}
where $\cL(\p)$ is given by:
\begin{eqnarray*}
\hspace{-1cm} \cL(\p) = - \frac{1}{\tau_{k-1}}\sum_{t=\tau_{k-1}}^{\tau_{k}-1} &\bigg\{\ind(y_t =1) \log (1-F(\f^{-1}(p_t) - \p\cdot \tx_t ))\nonumber \\
& + \ind(y_t =-1) \log (F(\f^{-1}(p_t) - \p \cdot \tx_t )) \bigg\}
\end{eqnarray*}
\vspace{-1.6cm}
\begin{align} \label{eq:log_likelihood2}
\phantom{a}\hspace{-1.2cm}
\end{align}
\\

\STATE For each period $t$ during the $k$-th episode, set
\begin{align}\label{eq:price2}
\hspace{-2.3cm} p_t \leftarrow \f(g_\f( \hp^k\cdot \tx_t)) \hspace{-1.9cm}
\end{align}

\ENDFOR
\end{algorithmic}
\caption{\bf RMLP Policy for dynamic pricing under the nonlinear setting}\label{alg-nonlinear}
\end{algorithm}

Our next theorem bounds the regret of our pricing policy (Algorithm~\ref{alg-nonlinear}).

\begin{theorem}\label{thm:nonlinear}
Let $\f$ be log-concave and strictly increasing. Suppose that Assumptions~\ref{ass1} and \ref{ass2-2} (or its alternative, Assumption \ref{ass2-3}) hold. Then, regret of the RMLP policy described as Algorithm~\ref{alg-nonlinear} is of $O(s_0 \log d \cdot \log T)$.
\end{theorem}

Proof of Theorem~\ref{thm:nonlinear} is given in Appendix~\ref{app:nonlinear}. Here, we summarize its key ingredients.
\begin{enumerate} 
\item  By increasing property of $\f$, a sale occurs at period $t$ when $z_t\ge \f^{-1}(p_t) - \p_0\cdot \tx_t$. 
Hence, the log-likelihood estimator for this setting reads as~\eqref{eq:log_likelihood2}. By virtue of Assumption~\ref{ass2-2} (or its alternative, Assumption~\ref{ass2-3}) we get a similar estimation error for the regularized estimator to the one in Proposition~\ref{thm:learning}.
\item Similar to our derivation for linear setting,  we show that the optimal pricing policy that knows $\p_0 = (\tth,\alpha_0)$ in advance is given by $p^*_t = \f(g_\f(\tth\cdot \tx_t))$, where $g_\f$ is defined based on Equation~\eqref{g:nonlinear}.
\item The difference between the posted price and the optimal price can be bounded as $p_t - p^*_t = \f(g_\f(\hp^k\cdot \tx_t)) - \f(g_\f(\p_0\cdot \tx_t))
\le L |\tx_t\cdot(\hp^k-\p_0)|$, for a constant $L>0$. This bound is similar to the corresponding bound for the linear setting, and following the same lines of our regret analysis for that case, we get $R(T)= O(s_0 \log d\cdot \log T)$. 
\end{enumerate} 
\section{Knowledge of market noise distribution}
The proposed RMLP policy has assumed that the market noise distribution $F$ is known to the seller. Knowledge of $F$ has been used both in estimating the model parameters $(\theta_0, \alpha_0)$ and in setting the prices $p_t$. On the other hand, the benchmark policy is also assumed to have access to model parameters and the distribution $F$. Therefore, the regret bound established in Theorem~\ref{thm:regret} essentially measures how much the seller loses in revenue due to lack of knowledge of the underlying model parameters. In practice, however, the underlying distribution of valuations is not given and this rises the question of \emph{distribution-independent} pricing policy.


It is worth mentioning that in some applications, although the underlying distribution of valuations is unknown, it belongs to a known class of distributions. For example, lognormal distributions have proved to be a good fit for the distribution
of valuations of advertisers in online advertising markets~\cite{king2007internet,lahaie2007revenue,xiao2009optimal,balseiro2014yield}. In Section~\ref{sec:unknown_ parameterized}, we consider a model where the underlying distribution belongs to a known class of log-concave distributions and propose a policy whose regret is $O(\sqrt{T})$. We also argue that no policy can get a better regret bound. 

Next,  we pursue pricing policies under completely unknown distribution. Here, the regret is measured against an optimal clairvoyant policy that has full knowledge of the model parameters $\p_0$ and market noise \emph{realizations}, $\{z_t\}_{t\ge 1}$, and thus extracts the customers' valuation at each step. Note that such a clairvoyant policy is much more powerful than the one considered in previous sections, as now it has access to noise realizations while before it only had knowledge of the noise distribution $F$.  

\subsection{Unknown distribution from a known class} \label{sec:unknown_ parameterized}
Suppose that the maket noises are generated from a log-concave distribution $F_{m,\sigma}$ (e.g., Log-normal), with unknown mean $m$ and unknown variance $\sigma^2$. Without loss of generality, we can assume that $m = 0$; otherwise, in the valuation model~\eqref{eq:model}, $m$ can be absorbed in the intercept term $\alpha_0$. We next explain how the RMLP policy can be adapted to this case.

Define $\beta_0 = 1/\sigma$ and consider the transformation $\tv_t = \beta_0 v_t$, $\ttheta_0 = \beta_0 \theta_0$, $\tila_0 = \beta_0 \alpha_0$, $\tz_t = \beta_0 z_t$. Then, the valuation model~\eqref{eq:model} can be written as
\begin{align}\label{eq:model3}
\tv_t = x_t \cdot \ttheta_0 + \tila_0 + \tz_t\,,
\end{align}
where $\tz_t$ are drawn from $F_{0,1}$. To lighten the notation, we use the shorthand $F \equiv F_{0,1}$. We also let $\mu_0 = (\ttheta_0,\tila_0)$. 
The response variables $y_t$ are then given by $y_t = \ind(\tv_t \ge \beta_0 p_t)$. 

 \begin{algorithm}[t]
\begin{algorithmic}[1]

\REQUIRE Pricing function $g$ (corresponding to $F_{0,1}$), regularizations $\lambda_k$, $\l1u$ (bound on $\|\p_0\|_1$)\hspace{2cm}
\REQUIRE{\bf (arrives over time)} covariate vectors $\{\tx_t = (x_t,1)\}_{t\in \naturals}$ 

\ENSURE prices $\{p_t\}_{t\in \naturals}$ 


\FOR{each episode $k = 1,2,\dots$}
\STATE For the first period of the episode, offer the price uniformly at random from $[0,1]$.

\STATE Denote by $\cA_k$ the set of first periods in episodes $1,\dots, k$.

\STATE Update the model parameter estimate $\hp^{k}$ using the regularized ML estimator:
\begin{align}\label{eq:ML}
\hspace{-4.5cm}  (\hp^{k},\hbeta^k)  = \underset{\|(\p/\beta,\beta)\|_1 \le \l1u}{\text{arg min}} \,\, \left\{ \cL(\beta,\p) + \lambda_k \|\p\|_1 \right\} \hspace{-3.5cm}
\end{align}
with 
\begin{align} \label{eq:log_likelihood}
\hspace{-8.7cm} \cL(\p,\beta) = - \frac{1}{k}\sum_{t\in \cA_k} \bigg\{\ind(y_t =1) \log (1-F(\beta p_t - \p\cdot \tx_t )) + \ind(y_t =-1) \log (F(\beta p_t - \p\cdot \tx_t )) \bigg\} 
\hspace{-6.8cm}
\end{align}
%
%

\STATE For each period $t$ during the $k$-th episode, set
\begin{align}\label{eq:price}
\hspace{-2cm} p_t \leftarrow \frac{1}{\hbeta^k} g(\hp^k\cdot \tx_t) \hspace{-1.7cm}
\end{align}

\ENDFOR
\end{algorithmic}
\caption{\bf RMLP-2 policy for dynamic pricing}\label{alg-RMLP2}
\end{algorithm}

We propose a variant of RMLP policy, called RMLP-2 for this case. Similar to RMLP, it runs in an episodic manner but the length of episodes grows linearly. (Episode $j$ is of length $j$ periods.) At the first period of each episode, the price is chosen randomly and independently from the feature vectors. To be concrete, we set the price uniformly at random from $[0,1]$. At the other periods of the episode, the price is set optimally based on the current estimate of the model parameters. Specifically, for episode $k$, we set $p_t = (1/\hbeta^k) g(\hp^k\cdot \tx_t)$,  where the pricing function $g$ is defined based on distribution $F\equiv F_{0,1}$, given by~\eqref{eq:g}, and the estimates $(\hp^k,\hbeta^k)$ are obtained via regularized log-likelihood. In forming the log-likelihood loss, we only consider the first period of each episode, where the prices are set randomly; for $k\ge 1$, we denote by $\cA_k$ the set of first periods in episodes $1,\dots, k$, and write the log-likelihood based on the samples in $\cA_k$: 
\begin{eqnarray}
 \cL(\p,\beta) = - \frac{1}{k}\sum_{t \in \cA_k} &\bigg\{\ind(y_t =1) \log \left(1-F(\beta p_t - \p\cdot \tx_t )\right)
 + \ind(y_t =-1) \log \left(F(\beta p_t - \p \cdot \tx_t )\right) \bigg\}
\end{eqnarray}
A formal description of RMLP-2 is given in Algorithm~\ref{alg-RMLP2}.  Note that in contrast to RMLP, in the RMLP-2 the length of episodes grows linearly rather than exponentially.
This ways, we have $|\cA_k| = k$, which provides enough samples to update the estimate $\hth^k$ at a proper rate to get regret $O(\sqrt{T})$. 

Our next result bounds the regret of RMLP-2.
\begin{theorem}\label{thm:regret2}
Consider the valuation model~\eqref{eq:model}, where noises $z_t$ are generated from a distribution $F_{m,\sigma}$, with unknown mean $m$ and variance $\sigma^2$.
Under Assumption \ref{ass2} and assuming that distribution $F_{m,\sigma}$ satisfies Assumption~\ref{ass1}, 
the regret of RMLP-2 policy is of $O\big(s_0 (\log d) \sqrt{T} \big)$. Further, regret of any pricing policy in this case is $\Omega(\sqrt{T})$.
\end{theorem}
We refer to Section~\ref{proof:regret2} for the proof of Theorem~\ref{thm:regret2}. As discussed in the proof, the lower bound $\Omega(\sqrt{T})$ applies to this case due to the existence of non-informative prices; See also Section~\ref{uninformative}.

\subsection{A distribution-independent pricing policy}
In this section, we propose a policy, called DIP (Distribution Independent Pricing), for the settings  that the underlying valuation distribution is completely unknown. Before a detailed description of DIP, we provide the general intuition behind this policy.

Here, our focus is on applications where signal-to-noise ratio is large. Specifically, we assume that the customer's valuations are given by model~\eqref{eq:model} and the noise terms $z_t$ are drawn from an unknown distribution with bounded support. (The support of distribution is considered to be small compared to the nominal valuations $\tx_t \cdot \theta_0+\alpha_0 $.) Therefore, valuations $v_t$ belong to a bounded interval $[0, K]$.
Similar to RMLP, the DIP policy operates in episodes. Each episode consists of an exploration phase followed by an exploitation phase.  All exploration phases are of length $c$, where $c\ge 1$ is a constant.
In these phases, the prices are set uniformly at random from the interval $[0,K]$. Following the exploration phase of episode $k$, there is an exploitation phase of $k$ periods. In this phase, we offer the optimal prices
based on the current estimates of the model parameters from the responses in the previous exploration phases. Therefore, the $k$-th episode consists of $(c+k)$ periods. In early episodes, the ratio of exploration phase to exploitation phase  
is high, as we know very little about the model parameters and then it becomes lower in the later episodes as we have already obtained a good estimate of the underlying model parameters.

The formal description of the DIP policy is given in Algorithm~\ref{alg-DIP}. Our focus is on bounded noise, i.e, $|z_t| \le \delta$ almost surely and hence we can take $K = \l1u + \delta$ as the bound on customer's valuations.

We next prove a regret guarantee for DIP policy.

\begin{theorem}[Regret Upper Bound]\label{thm:regret3}
Consider the valuation model~\eqref{eq:model}, where the noise terms $\{z_t\}_{t\ge 1}$ are generated from an unknown zero-mean distribution with support $[-\delta,\delta]$. Further, suppose that the feature vectors satisfy Assumptions~\ref{ass2}.
Then, the regret of the DIP policy is $O(s_0 (\log d) \sqrt{T} + \delta T)$. Here, the regret is against an optimal clairvoyant policy that knows the model parameters and the noise realizations $\{z_t\}_{t\ge 1}$.
\end{theorem}
\begin{algorithm}[t]
\begin{algorithmic}[1]

\REQUIRE exploration length $(c)$, regularizations $\lambda_k$, $\l1u$ (bound on $\|\p_0\|_1$), noise bound $\delta$ \hspace{2cm}
\REQUIRE{\bf (arrives over time)} covariate vectors $\{\tx_t = (\phi(x_t),1)\}_{t\in \naturals}$ 

\ENSURE prices $\{p_t\}_{t\in \naturals}$ 

\STATE $K \leftarrow \l1u +\delta$

\FOR{each cycle $k = 1, 2,3,\dots$}
\STATE \emph{Exploration episode} ($c$ periods): Offer prices uniformly at random from $[0,K]$.

\STATE Update the model parameter estimate $\hp^{k}$ using the regularized ML estimator obtained\\ from observations during the previous exploration episodes:
\begin{align}\label{eq:ML4}
\hspace{-3.8cm}  \hp^{k}  = \underset{\|\p\|_1 \le \l1u}{\text{arg min}} \,\, \left\{ \cL(\p) + \lambda_k \|\p\|_1 \right\} \hspace{-2.4cm}
\end{align}
where $\cL(\p)$ is given by:
\begin{align}\label{quad-loss}
\hspace{-3cm} \cL(\p) =  \frac{1}{ck} \sum_{t\in \cA_k} (K y_{t} - \tx_t\cdot \p )^2\,,\hspace{-3.1cm}
\end{align}
and $\cA_k$ denotes the set of periods belonging to the first $k$ exploration episodes.

\STATE \emph{Exploitation episode} ($k$ periods): Offer prices based on the current estimate $\hp^{k}$ as 
\begin{align}\label{eq:price3}
\hspace{-2.3cm} p_t \leftarrow \hp^k\cdot \tx_t -2\delta \hspace{-1.5cm}
\end{align}

\ENDFOR
\end{algorithmic}
\caption{\bf Distribution Independent Pricing (DIP) Policy}\label{alg-DIP}
\end{algorithm}

In the following, we outline the main idea of the proof of Theorem~\ref{thm:regret3}. The proof minutiae are deferred to Section~\ref{proof:thm-regret3}.

For a given time $T$, it is easy to verify that the number of cycles up to time $T$ is $O(\sqrt{T})$.  Recall that in the exploration phases the prices are set randomly. The regret incurred in each period is $O(1)$ since the valuations are bounded. Therefore, the cumulative regret in the exploration phases up to time $t$ is $O(\sqrt{T})$.  Next, we bound the regret incurred during the exploitation phases. For each episode $k$, prices are posted as $p_t = \hp^k\cdot \tx_t -2\delta$. Note that the term $2\delta$ is to ensure purchases occur with high probabilities. The regret is then due to the conservative term $2\delta$ and the estimation error $\tx_t \cdot(\hp^k-\p_0)$.
The aggregate effect of these two factors results in a total regret of $O(\delta k + s_0\log d)$ in episode $k$. Since there are $O(\sqrt{T})$ cycles up to time $T$, the total regret incurred during the exploitation episodes is $O(\delta T + s_0(\log d)  \sqrt{T})$.

\section{Proof of Theorems}
\subsection{Proof of Theorem~\ref{thm:regret}}\label{proof:thm2}


Following step 1 of the proof outline mentioned in Section~\ref{sec:regret}, we consider the problem of estimating $\p_0$ based on observations from previous episode.
{Before we proceed, let us emphasize once again that the way RMLP is designed, posted prices at each episode are statistically independent from the market noises in that episode. This can be 
easily observed because $p_t = g(x_t\cdot \hp^k)$ for $t$ belonging in the $k$-th episode, and $\hp^k$ is estimated based on the samples in the $(k-1)$-th episode.}

{We fix $k\ge 1$ and to lighten the notation, we use the indices ${1,2,\dotsc,n}$ to correspond to periods in the $k$-the episode, i.e., $t = \tau_{k}, \tau_{k}+1,\dotsc, \tau_{k+1}-1$.} 

Using probabilistic model~\eqref{eq:prob-model}, $\p_0$ is estimated by solving a regularized maximum likelihood (ML) optimization problem. The (normalized) negative log-likelihood function for $\p$ reads as
\begin{align}\label{eq:loss}
\cL(\p) = - \frac{1}{n}\sum_{t=1}^{n} \bigg\{\ind(y_t =1) \log (1-F(p_t - \p\cdot \tx_t )) + \ind(y_t =-1) \log (F(p_t - \p\cdot \tx_t )) \bigg\}\,.
\end{align}
Parameter $\p$ is estimated as the solution of the following program:
\begin{align}\label{eq:ML2}
\hp  = \underset{ \|\p\|_1\le \l1u}{\text{arg min}} \,\, \cL(\p) + \lambda \|\p\|_1
\end{align}
Define $\ell_{\l1u}$ as follows which corresponds to ``flatness" of function $\log F$:
\begin{align}\label{eq:lM}
\ell_{\l1u} \equiv \inf_{|x|\le 3\l1u} \left\{\min \Big\{-\log''F(x)  , -\log''(1-F(x))  \Big\}\right\}\,.
\end{align}
By Assumption~\ref{ass1}, the log-concavity property of $F$ and $1-F$, we have $\ell_{\l1u}>0$. 


The next theorem upper bounds the estimation error of the proposed regularized estimator. 
\begin{proposition}[Estimation Error]\label{thm:learning}
Consider linear model~\eqref{eq:model} with $\p_0 = (\tth,\alpha_0) \in \Omega$, under Assumptions~\ref{ass1} and~\ref{ass2}. 
Let $\hp$ be the solution of optimization problem~\eqref{eq:ML2} with $\lambda \ge 4u_\l1u \sqrt{{(\log d)}/{n}}$. Then, there exist positive constants $c_0$ and $C$ such that, for $n\ge c_0 s_0 \log (d)$, the following inequality holds
 with probability at least $1- 1/d - 2e^{-n/(c_0s_0)}$:
\begin{align}\label{L2B}
\|\hp - \p_0\|_2^2\le \frac{16 s_0 \lambda^2}{\ell_{\l1u}^2 C_{\min}^2}\,.
\end{align}
\end{proposition}
We refer to Appendix~\ref{app:RE} for the proof of Proposition~\ref{thm:learning}.

As we see the $\ell_2$ estimation error scales linearly with the sparsity level $s_0$. As $s_0$ increases, the number of parameters to be estimated becomes larger and this makes the estimation problem harder, leading to worse $\ell_2$ bound for a fixed number of samples, $n$. Further, choosing $\lambda \sim \sqrt{(\log d)/n}$ (where $\sim$ indicates equality up to a constant factor), our $\ell_2$ bound scales logarithmically in the dimension of the demand space, $d$. This allows to  deal with high-dimensional applications and obtain a regret that scales logarithmically in $d$. Further, the estimation error shrinks as $\sim1/n$; getting more samples with fixed value of $s_0$ and $d$ leads to better estimation accuracy. Finally, note that for small values of $\ell_{\l1u}$, the log-likelihood function is very flat and there can be, in principle, vectors $\p$ of log-likelihood value very close to the optimum and nevertheless far from the optimum. In other words, estimation task becomes harder as $\ell_{\l1u}$ gets smaller and this is clearly reflected in the derived estimation bound. 

We next use Proposition~\ref{thm:learning} to bound the expected estimation error. 
\begin{coro}\label{coro:E1}
Under assumptions of Proposition~\ref{thm:learning}, the following holds true:
\begin{align}\label{eq:E1}
\E(\|\hp-\p_0\|_2^2) \le \frac{16s_0\lambda^2}{\ell_{\l1u}^2 C_{\min}^2} + 4\l1u^2  \left(\frac{1}{d} + 2e^{-n/(c_0s_0)}\right)\,.
\end{align}
\end{coro}
Proof of Corollary~\ref{coro:E1} is straightforward and is omitted.

In the next proposition, we improve bound~\eqref{eq:E1} for $n\ge c_1d$, for a constant $c_1>0$. As we will see, the following result is useful to develop sharper upper bound for regret of RMLP policy. 
\begin{proposition}\label{pro:E2}
Under assumptions of Proposition~\ref{thm:learning}, there exist constants $c, c_1>0$, such that for $n \ge  c_1d$, the following holds true:
\begin{align}
\E(\|\hp-\p_0\|_2^2) \le \frac{16(s_0+1)\lambda^2}{\ell_{\l1u}^2 C_{\min}^2} + 4\l1u^2 e^{-cn^2}\,.
\end{align}
\end{proposition}
Proposition~\ref{pro:E2} is proved in Appendix~\ref{proof:pro-E2}.

We next establish some useful properties of the virtual valuation function $\varphi$ and the price function $g$.
\begin{lemma}\label{lem:phi-monotone}
If $1-F$ is log-concave, then the virtual valuation function $\varphi$ is strictly monotone increasing.  
\end{lemma}
\begin{lemma}\label{lem:g-Lipschitz}
If $1-F$ is log-concave, then the price function $g$ satisfies $0<g'(v)< 1$, for all values of $v\in \reals$.   
\end{lemma}
Proofs of Lemma~\ref{lem:phi-monotone} and~\ref{lem:g-Lipschitz} are given in Appendix~\ref{app:phi-monotone} and~\ref{app:g-Lipschitz}, respectively.

Given that $\|\hp^k\|_1\le \l1u$ and $|\tx_t\cdot \hp^k|\le \l1u$ for all $t, k$,  
\begin{align}\label{P}
p_t  = g(\tx_t\cdot \hp^k) &\le 2 |\tx_t\cdot \hp^k|   \le 2\l1u\,,
\end{align}
where in the first inequality we used the fact that $\varphi(v)$ is increasing as per Lemma~\ref{lem:phi-monotone} and hence $g(v) = v+\varphi^{-1}(-v) \le v+|v| \le 2|v|$. Similarly, we have $p^*_t \le 2\l1u$ for all $t$. 


%
We are now ready to bound the regret of our policy.
For $t\ge 1$, let
\begin{align}
R_t\equiv p^*_t \ind(v_t \ge p^*_t) - p_t \ind(v_t \ge p_t)
\end{align}
be the regret at period $t$.
Further, let $\cH_t = \{x_1, x_2, \dotsc, x_{t}, z_1, z_2, \dotsc, z_t\}$ be the history set, up to time $t$ (more precisely, $\H_t$ is the filtration generated by $\{x_1, x_2, \dotsc, x_{t}, z_1, z_2, \dotsc, z_t\}$). We also define $\bcH_t = \cH_t \cup \{x_{t+1}\}$ as the filtration obtained after augmenting by the new feature $x_{t+1}$. 

We write
\begin{align}
\E(R_t|\bcH_{t-1}) &= \E(p^*_t \ind(v_t \ge p^*_t)| \bcH_{t-1}) -\E(p_t \ind(v_t \ge p_t)| \bcH_{t-1})\\
& =p^*_t(1-F(p^*_t-\tx_t\cdot \p_0)) - p_t(1-F(p_t-\tx_t\cdot \p_0))\label{eq:Rt1}
\end{align}
Define $r_t(p) \equiv p(1-F(p-\tx_t\cdot \p_0))$ as the expected revenue under price $p$. 
Note that $p^*_t\in \arg\max r_t(p)$ and thus $r'_t(p^*_t) = 0$.
By Taylor expansion, 
\begin{align}\label{eq:Rt2}
r_t(p_t) = r_t(p^*_t) + \frac{1}{2} r''_t(p) (p_t-p^*_t)^2\,,
\end{align}
for some $p$ between $p_t$ and $p^*_t$. 

We next show that $|r''_t(p)|\le C$, with $C = 2(B+\l1u B')$, $B = \max_{v} f(v)$, and $B' = \max_v f'(v)$. To see this, we write
\begin{align}\label{eq:Rt3}
|r''_t(p)| = |2f(p-\tx_t\cdot \p_0) + pf'(p-\tx_t\cdot \p_0)| \le 2B+ 2\l1u B' = C\,,
\end{align}
where we use the fact that $p_t,p^*_t \le 2\l1u$ and consequently $p\le 2\l1u$.

Combining Equations~\eqref{eq:Rt1}, \eqref{eq:Rt2}, \eqref{eq:Rt3}, along with 1-Lipschitz property of $g$ gives
\begin{align}\label{eq:step0}
\E(R_t|\bcH_{t-1}) \le \frac{C}{2} (p^*_t-p_t)^2= \frac{C}{2} (g(\p_0\cdot \tx_t)-g(\hp^k\cdot \tx_t))^2 
\le \frac{C}{2} |\tx_t\cdot(\p_0-\hp^k)|^2\,.
\end{align}
Given that $\tx_t$ is independent of ${\cH}_{t-1}$, we have
\begin{align}
\E(R_t|{\cH}_{t-1}) \le \frac{C}{2}  \<\hp^k-\p_0, \tSigma (\hp^k-\p_0)\> \,, \label{eq:Rt4-00}
\end{align}
where $\tSigma = \E(\tx_t \tx_t^\sT)$. Using Equation~\eqref{eq:sigma-aug},
\begin{align}
\E(R_t) = \E\left(\E(R_t|\cH_{t-1})\right) \le \frac{1}{2}{CC_{\max}} \E(\|\hp^k-\p_0\|_2^2)\,.  \label{eq:Rt4}
\end{align}

Now, since the length of episodes grows exponentially, the number of episodes by period $T$ is logarithmic in $T$. Specifically, $T$ belongs to episode $K = \lfloor \log T \rfloor+1$.
Hence,
\begin{align}\label{tot-regret}
\Reg(T) = \sum_{k=1}^{K} \Reg (k{\rm th \,\,\,Episode}) 
\end{align}

We bound the total regret over each episode by considering three separate cases:
\begin{itemize}
\item $2^{k-2} \le c_0 s_0\log d$: Here, $c_0$ is the constant
in the statement of Proposition~\ref{thm:learning}. In this case, episodes are not large enough to estimate $\p_0$ accurately enough, and thus we use a naive bound on regret.
Clearly, by~\eqref{P}, we have $\E(R_t) \le p_t^* \le 2\l1u$. Since the length of $k$ th episode is $2^{k-1} \le 2c_0 s_0\log d$, the total regret incurred during episode $k$ is at most 
$4c_0 \l1u s_0\log d$.

\item  $c_0 s_0\log d \le 2^{k-2} \le c_1 d$: Here, $c_1$ is the constant in the statement of Proposition~\ref{pro:E2}. Continuing from Equation~\eqref{eq:Rt4} and applying Corollary~\ref{coro:E1} to episode $k$, we obtain 
\begin{align}
\Reg (k{\rm th \,\,\,Episode}) &= \sum_{t = \tau_k}^{\tau_{k+1}-1} \E(R_t)\nonumber\\
&\le \frac{1}{2} C C_{\max}\sum_{t = \tau_k}^{\tau_{k+1}-1} \E(\|\hp^k -\p_0\|_2^2)\nonumber\\
&\le \frac{1}{2} C C_{\max}\left\{ \frac{16s_0\lambda_k^2}{\ell_{\l1u}^2 C_{\min}^2} \tau_k + 4\l1u^2  \left(\frac{\tau_k}{d} + 2\tau_k e^{-\tau_{k-1}/(c_0s_0)}\right) \right\}\nonumber\\
&\le \frac{1}{2} C C_{\max}\left\{\left(\frac{16 u_{\l1u}}{\ell_{\l1u} C_{\min}}\right)^2 2s_0 \log d + 8\l1u^2\left(2c_1 + 2\tau_{k-1}e^{-\frac{\tau_{k-1}}{c_0s_0}} \right)\right\}\,,\label{eq:Case2-1}
\end{align}
where in the last step we used $\tau_k = 2\tau_{k-1}$ and $\tau_k = 2^{k-1} \le 2c_1 d$. Therefore, in this case
\begin{align}
\Reg (k{\rm th \,\,\,Episode}) \le \frac{C'}{C_{\min}^2} s_0 \log d\,,
\end{align}
where $C'$ hides various constants in the right-hand side of~\eqref{eq:Case2-1}. 

\item  $c_1 d<  2^{k-2}$: Continuing from Equation~\eqref{eq:Rt4} and applying Proposition~\ref{pro:E2} to episode $k$, we obtain 
\begin{align}
\Reg (k{\rm th \,\,\,Episode}) &= \sum_{t = \tau_k}^{\tau_{k+1}-1} \E(R_t)\nonumber\\
&\le \frac{1}{2} C C_{\max}\sum_{t = \tau_k}^{\tau_{k+1}-1} \E(\|\hp^k -\p_0\|_2^2)\nonumber\\
&\le \frac{1}{2} C C_{\max}\left\{ \frac{16(s_0+1)\lambda_k^2}{\ell_{\l1u}^2 C_{\min}^2} \tau_k + 4\tau_k \l1u^2 e^{-c\tau_{k-1}^2} \right\}\nonumber\\
&\le \frac{1}{2} C C_{\max} \left\{ \left( \frac{16 u_{\l1u}}{\ell_{\l1u} C_{\min}}\right)^2 2(s_0+1) \log d   + 8\l1u^2 \tau_{k-1} e^{-c\tau_{k-1}^2}  \right\}\,,\label{eq:Case3-1}
\end{align}
Therefore, in this case
\begin{align}
\Reg (k{\rm th \,\,\,Episode}) \le \frac{C'}{C_{\min}^2} s_0 \log d\,,
\end{align}
where $C'$ hides various constants in the right-hand side of~\eqref{eq:Case3-1}. 

\end{itemize}
Combining the above three cases into Equation~\eqref{tot-regret}, we get
\begin{align}\label{tot-regret2}
\Reg(T) \le K \frac{C'}{C_{\min}^2} s_0 \log d = O\Big(\frac{1}{C_{\min}^2} s_0 \log d \cdot \log T\Big)\,,
\end{align}
which concludes the proof.
\subsection{Proof of Theorem~\ref{thm:NoC}}\label{proof:thm-NoC}
By using Equation~\eqref{eq:Rt4-00}, we have 
\begin{align}
\E(R_t)\le \frac{C_1}{2 } \E\left(\left\<\hp^k - \p_0, \tSigma(\hp^k - \p_0)\right\>\right)\,,
\end{align}
with $\tSigma = \E(\tx_t \tx_t^\sT)$.

Therefore, letting $K = \lfloor \log T \rfloor +1$, 
\begin{align}\label{Reg-Noc}
\Reg(T) = \sum_{k=1}^{k_1} \Reg (k{\rm th \,\,\,Episode}) \le \frac{C_1}{2 } \sum_{k=1}^{k_1} \E\left(\left\<\hp^k - \p_0, \tSigma(\hp^k - \p_0)\right\>\right) \tau_k
\end{align}

We next bound the right-hand side of the above bound. Let $X^{(k)} \in \reals^{\tau_k\times d}$ be the matrix obtained by stacking feature vectors in episode $k$ as rows.
By applying bound~\eqref{eq:basic} to samples in episode $(k-1)$, we get that with probability at least $1 - 1/d$,  
\begin{align}
\frac{2\ell_{\l1u}}{\tau_{k-1}} \Big\|\tX^{(k-1)} (\p_0 - \hp^k)\Big\|^2 + 2\lambda_k \|\hp^k\|_1\le \lambda_k \|\hp^k - \hp_0\|_1 + 2\lambda_k \|\p_0\|_1
\end{align}
Hence,
\begin{align}\label{eq:SB}
\frac{2\ell_{\l1u}}{\tau_{k-1}} \Big\|\tX^{(k-1)} (\p_0 - \hp^k)\Big\|^2 \le  \lambda_k \|\hp^k - \hp_0\|_1 + 2\lambda_k \|\p_0\|_1 -2\lambda_k \|\hp^k\|_1
\le 3 \lambda_k \|\hp^k - \hp_0\|_1
\end{align}

For $k\ge 1$, let $S^{(k)}\in \reals^{d\times d}$ be the empirical covariance of $X^{(k)}$, and define $E^{(k)} = \tSigma -S^{(k)}$. Then, 
\begin{align}
\left\<\hp^k - \p_0, \tSigma(\hp^k - \p_0)\right\> = \left\<\hp^k - \p_0, S^{(k-1)}(\hp^k - \p_0)\right\> +  \left\<\hp^k - \p_0, E^{(k-1)}(\hp^k - \p_0)\right\>\,.
\end{align}

The first term is bounded using Equation~\eqref{eq:SB} as follows:
\begin{align}\label{eq:S-term1}
\left\<\hp^k - \p_0, S^{(k-1)}(\hp^k - \p_0)\right\> = \frac{1}{\tau_{k-1}} \Big\|\tX^{(k-1)} (\p_0 - \hp^k)\Big\|^2 \le \frac{3\lambda_k}{2\ell_{\l1u}} \|\hp^k - \hp_0\|_1\le \frac{3\l1u}{\ell_{\l1u}} \lambda_k\,,
\end{align}
with probability at least $1 - 1/d$.

The second term can be bounded by virtue of the following lemma, whose proof if deferred to Appendix~\ref{proof:lem-S-term2}
\begin{lemma}\label{lem:S-term2}
For any $k\ge 1$ and any vector $v\in \reals^d$, we have 
\[\<v, E^{(k)} v\>  \le 3 \sqrt{\frac{\log d}{\tau_k}} \, \|v\|_1^2\,, \]
with probability at least $1 - 8/d^2$.
\end{lemma}
By Lemma~\ref{lem:S-term2}, we have
\begin{align}\label{eq:S-term2}
\left\<\hp^k - \p_0, E^{(k-1)}(\hp^k - \p_0)\right\> \le 8 \sqrt{\frac{\log d}{\tau_{k-1}}} \l1u^2\,,
\end{align}
with probability at least $1 - 8/d^2$.

Combining Equations~\eqref{eq:S-term1} and \eqref{eq:S-term2}, with probability at least $1 - 9/d$ we have
\begin{align}\label{eq:SB2}
\left\<\hp^k - \p_0, \tSigma(\hp^k - \p_0)\right\> \le \frac{3\l1u}{\ell_{\l1u}} \lambda_k + 8 \sqrt{\frac{\log d}{\tau_{k-1}}} \l1u^2 \le C \sqrt{\frac{\log d}{\tau_{k-1}}}\,,
\end{align}
for some constant $C>0$. 

Following a similar argument as in Section~\ref{proof:thm2} (see Equation~\eqref{tot-regret} and onwards) we have that the following holds for a suitable constant $C> 0$: 
\begin{align*}
\Reg(T) &\le C \sum_{k=2}^{K} \sqrt{\frac{\log d}{\tau_{k-1}}} \cdot \tau_k\\
&\le \sqrt{2} C \sum_{k=2}^{K} \sqrt{(\log d){\tau_{k-1}}} = O(\sqrt{(\log d) T})\,,
\end{align*} 

\subsection{Proof of Theorem~\ref{thm:LB-regret}}\label{proof:LB-regret}
The regret benchmark~\eqref{eq:Regret_def} is defined as the maximum gap between a policy and the oracle policy
over different $\p_0\in\Omega$ and $p_X\in Q(\cX)$.
Without loss of generality, we assume $\cX = [-1,1]^d$.
In order to obtain a lower bound on the regret, it suffices to consider a specific distribution in $Q(\cX)$.
We consider a distribution $p_X$ that selects coordinates $x_i$, $1\le i\le d$, uniformly at random from $\{-1,1\}$ and independent of each other. We further assume that $\alpha_0 = 0$ and 
$\tth\in \Omega_0$, where $$\Omega_0 = \{\theta\in\reals^d:\, (\theta,0)\in \Omega\}\,.$$

Fix an arbitrary policy $\pi$ in family $\Pi$. Since the assumption  $\alpha=0$ is known to the oracle, we have $\pi(p_t) = g(x_t \cdot \theta_t)$, for some $\theta_t\in \Omega_0$, which is $\cH_{t-1}$-measurable . Recalling our notation in the proof of Theorem~\ref{thm:regret}, $R_t$ denotes the regret occurred at step $t$ and by Equations~\eqref{eq:Rt1}, \eqref{eq:Rt2}, we have
\begin{align}
\E(R_t|\bcH_{t-1}) = r_t(p^*_t) - r_t(p_t) = -\frac{1}{2} r''_t(p) (p_t-p^*_t)^2\,,\label{LB--1}
\end{align}
for some $p$ between $p_t$ and $p^*_t$.

Our first lemma will be used in lower bounding $\E(R_t|\bcH_{t-1})$.
\begin{lemma}\label{lem:hessian}
There exists a constant $c_1>0$ (depending on $\l1u$ and $\sigma$) such that, with probability one\footnote{The randomness comes from randomness in prices which in turn comes from randomness in features $x_t$.}, $r''_t(p^*_t) \le -c_1$, for all $t\ge 1$. Further, there exists constant $\delta>0$
(depending on $\l1u$ and $\sigma$) such that $r''_t(p) \le -c_1/4$ for $p\in [p^*_t-\delta, p^*_t+\delta]$, with probability one. 
\end{lemma}
Proof of Lemma~\ref{lem:hessian} is given in Appendix~\ref{proof:hessian}.

Continuing from Equation~\eqref{LB--1}, we consider two separate cases:
\begin{itemize}
\item $|p_t-p^*_t|\le \delta$: We have $p\in [p^*_t-\delta, p^*_t+\delta]$ and therefore by applying Lemma~\ref{lem:hessian} we obtain
\begin{align}
\E(R_t|\bcH_{t-1}) = r_t(p^*_t) - r_t(p_t) \ge \frac{c}{8} (p_t-p^*_t)^2\,.\label{LB--2}
\end{align}
\item $|p_t-p^*_t| > \delta$: Since function $r_t$ has only one local maximum, namely $p^*_t$, the function is increasing before $p^*_t$ and decreasing afterward.
Therefore, if $p_t\le p^*_t-\delta$ then
\begin{align}
r_t(p_t)\le r_t(p^*_t -\delta)  = r_t(p^*_t)+ \frac{1}{2} r_t''(p) \delta^2 \le r_t(p^*_t) - \frac{c_1}{8} \delta^2\,,
\end{align}
where $p$ is some point in $[p^*_t-\delta, p^*_t]$ and we applied Lemma~\ref{lem:hessian} in the last step.

Similarly, for $p_t\ge p^*_t+\delta$ we obtain
\begin{align}
r_t(p_t)\le r_t(p^*_t +\delta)  = r_t(p^*_t)+ \frac{1}{2} r_t''(p) \delta^2 \le r_t(p^*_t) - \frac{c_1}{8} \delta^2\,,
\end{align}
where $p\in [p^*_t-\delta, p^*_t]$ this time. Combining these two inequalities, we get that $r_t(p^*_t)-r_t(p_t)\ge c_1\delta^2/8$, if $|p^*_t-p_t|\ge \delta$.
\end{itemize}

Writing the bounds in the two cases together, we get

\begin{align}
\E(R_t|\bcH_{t-1}) \ge 
r_t(p^*_t) - r_t(p_t) \ge
\begin{cases}
\dfrac{c_1}{8} (p_t-p^*_t)^2\,,  &\text{ if }|p_t-p^*_t|\le \delta\,,\\
\\
\dfrac{c_1}{8}\delta^2\,, &\text{ if }|p_t-p^*_t| > \delta\,.
\end{cases}
\end{align}

We proceed by relating the lower bound to the error in estimation $\tth$.
\begin{align}
\E(R_t|\bcH_{t-1}) &\ge \frac{c_1}{8} \min\left((p_t-p^*_t)^2,\delta^2\right)  = \frac{c_1}{8} \min\Big((g(x_t\cdot \th_t)-g(x_t\cdot\tth))^2,\delta^2\Big)\nonumber\\
& \ge \frac{c_1 }{8} \min\Big( c_2^2\,|x_t\cdot(\th_t-\tth)|^2,\delta^2\Big)\,, \label{eq:l2risk-0}
\end{align}
where we used the fact that by Lemma~\ref{lem:g-Lipschitz}, $g'(v)>c_2$ over the bounded interval $[-\l1u,\l1u]$, for some constant $c_2>0$.
We recall the definition of history set $\cH_t \equiv \bcH_t\backslash\{x_{t+1}\}= \{x_1,x_2,\dotsc,x_t, z_1, z_2, \dotsc, z_t\}$. Since $\cH_{t}\subseteq \bcH_{t}$, by
iterated law of expectation, we get
\begin{align}
\E(R_t|\cH_{t-1}) &= \E(\E(R_t|\bcH_{t-1})|\cH_{t-1}) \ge
\frac{c_1}{8} \E\Big(\min\big(c_2^2 |x_t\cdot(\th_t-\tth)|_2^2,\delta^2\big)\Big|\cH_{t-1}\Big)\label{eq:l2risk-1}
\end{align}
Note that $x_t$ is independent of $\cH_{t-1}$ and $\th_t-\tth$ is $\cH_{t-1}$-measurable.

We use the following lemma to lower bound the right-hand side of~\eqref{eq:l2risk-1}.  
\begin{lemma}\label{lem:Emin}
Let $x \in \reals^d$ be a random vector such that its coordinates are chosen independently and uniformly at random from $\{-1,1\}$. Further, suppose that $v\in \reals^d$ and $\delta>0$ are deterministic. Then, 
\begin{align}
\E\Big(\min\big((x\cdot v)^2,\delta^2\big)\Big) \ge 0.1 \min(\|v\|_2^2,\delta^2)\,.
\end{align}
\end{lemma}
Proof of Lemma~\ref{lem:Emin} is given in Appendix~\ref{proof:Emin}.

Applying Lemma~\ref{lem:Emin} to bound~\eqref{eq:l2risk-1}, we obtain
\begin{align}
\E(R_t|\cH_{t-1})\ge
\frac{c_1c_2^2 }{80} \E\Big(\min\Big(\|\th_t-\tth\|_2^2, {\delta^2}/{c_2^2}\Big)\Big|\cH_{t-1} \Big) \,.\label{eq:l2risk-2}
\end{align}
Now, taking expectation from both sides with respect to $\cH_{t-1}$, we arrive at
\begin{align}
\E(R_t)\ge
\frac{c_1c_2^2 }{80} \E\Big(\min\Big(\|\th_t-\tth\|_2^2, {\delta^2}/{c_2^2}\Big)\Big) \,.\label{eq:l2risk-2-2}
\end{align}
Equation~\eqref{eq:l2risk-2-2} lower bounds the expected regret at each step to the $\ell_2$ estimation error.

We continue by establishing a minimax lower bound on $\ell_2$-risk of estimation.

\begin{lemma}\label{lem:Fano}
Consider linear model~\eqref{eq:model}, with $\alpha_0 = 0$, and assume that the market values $v(x_t)$, $1\le t\le T$, are fully observed and the feature vectors are generated according to $p_X$, described above. We further assume that the noise in market value is generated as $z_t\sim \normal(0,\sigma^2)$. For a sequence of estimators $\theta_t$, we let $\th_1^t = (\th_1, \th_2, \dotsc, \th_t)$. Then, conditional on feature vectors  $(x_1,\dotsc, x_T)$,  and for any fixed value $C>0$, there exists
a nonnegative constant $\tC$, depending on $C$, $\sigma$, $\l1u$, such that
\begin{align}
\min_{\th_1^T}\, \max_{\tth\in \Omega_0} \,\sum_{t=1}^T \E\left(\min\left( \|\th_t - \tth\|^2_2,C\right)\right) \ge 
\tC  \bigg\{ s_0 \log \Big(\frac{T}{s_0}\Big)+ \min\bigg[\frac{T}{s_0},  s_0 \log \Big(\frac{d}{s_0}\Big)\bigg]\bigg\} \,.
\end{align}
\end{lemma}
Proof of Lemma~\ref{lem:Fano} is given in Appendix~\ref{App:Fano}.

We are now ready to lower bound the regret of any policy in $\Pi$. 
\begin{align}
\Reg(T) &\ge \max_{\tth\in \Omega_0} \sum_{t=1}^T \E(R_t) \ge \frac{c_1c_2^2 }{80} \sum_{t=1}^T \E\Big(\min\Big(\|\th_t-\tth\|_2^2, {\delta^2}/{c_2^2}\Big)\Big)\\
&\ge \tC \frac{c_1c_2^2}{80}  \bigg\{ s_0 \log \Big(\frac{T}{s_0}\Big)+ \min\bigg[\frac{T}{s_0},  s_0 \log \Big(\frac{d}{s_0}\Big)\bigg]\bigg\}  \,.
\end{align}
where the last step follows from Lemma~\ref{lem:Fano}.

\subsection{Proof of Theorem~\ref{thm:nonlinear}}\label{app:nonlinear}
Let $\tx_t = (\phi(x_t),1)$ denote the transformed features under the feature-map, augmented by the constant term $1$. Also, let  $\tp_t  =\psi^{-1}(p_t)$. We first show that Assumption~\ref{ass2-3} implies Assumption~\ref{ass2-2},
and therefore it suffices to prove the theorem under Assumption~\ref{ass2-2}.
\begin{lemma}\label{lem:ass-equiv}
Suppose that Assumption~\ref{ass1} hold true. Then, Assumption~\ref{ass2-3} implies Assumption~\ref{ass2-2}.
\end{lemma}
Proof of Lemma~\ref{lem:ass-equiv} is given in Appendix~\ref{app:ass-equiv}.

By Assumption~\ref{ass1}, the support of $\prob_X$ is abounded set $\cX$. Given that $\phi$ has a continuous derivative, it is Lipschitz on the bounded set $\cX$ and ergo the image of $\cX$ remains bounded under the feature-map $\phi$. Putting differently, features $\tx_t$ are sampled from a bounded set in $\reals^d$. Without loss of generality, we assume $\|\tx_t\|_\infty \le 1$. Further, as per  Assumption~\ref{ass2-2}, the covariance of the underlying distribution $\Sigma_\phi$ is positive definite with bounded eigenvalues. 

On a different note, since $\f$ is strictly increasing, a sale occurs at period $t$ when $\p_0 \cdot \tx_t +z_t \ge \f^{-1}(p_t) =\tp_t $. Therefore the (negative) log-likelihood function for $\p$ reads as
\begin{eqnarray*}
 \cL(\p) = - \frac{1}{\tau_{k-1}}\sum_{t=\tau_{k-1}}^{\tau_{k}-1} \bigg\{\ind(y_t =1) \log (1-F(\tp_t - \p\cdot \tx_t )) + \ind(y_t =-1) \log (F(\tp_t - \p\cdot \tx_t )) \bigg\}\,. 
\end{eqnarray*}
The estimation bound~\eqref{L2B} also holds for this setting and the proof goes along the same lines of the proof of Proposition~\ref{thm:learning}, with slight modifications: $(i)$ the features $x_t$ and prices $p_t$ should be replaced by $\tx_t$ and $\tp_t$. $(ii)$ Quantity $u_{\l1u}$ and $\ell_{\l1u}$ in the statement of Propostion~\ref{thm:learning} should be set as $M = (1/3)g_\f(0)+ (2/3) \l1u$. This follows from the bounds below 
\begin{align}
\tp_t = g_\f(\tx_t\cdot \hp^k) &\le g_\f(0) + |\tx_t\cdot \hp^k| \le
g_\f(0) + \l1u\,.\label{M:nonlinear}
\end{align}
Here, we used the facts that $g_\f$ is $1$-Lipschitz and increasing as explained below Equation~\eqref{g:nonlinear}.

We next characterize the optimal policy when the true parameter $\p_0 = (\tth,\alpha_0)$ is known.
The expected revenue from a poster price $p$ works out at  $p(1-F(\f^{-1}(p)-\p_0\cdot \tx_t))$. Writing this in terms of $\tp= \f^{-1}(p)$, the first order condition for 
the optimal price reads as 
\begin{align}
\lambda (\tp^*-\p_0\cdot \tx_t) \equiv \frac{f(\tp^*-\p_0\cdot \tx_t)}{1-F(\tp^*-\p_0\cdot \tx_t)} = \frac{\f'(\tp^*)}{\f(\tp^*)}\,,
\end{align}
where $\lambda$ denotes the hazard rate function.
Equivalently 
\begin{align}
\p_0\cdot \tx_t = \tp^* -\lambda^{-1}\Big(\frac{\f'(\tp^*)}{\f(\tp^*)}\Big)\,.
\end{align}
By definition of function $g_\f$ as per Equation~\eqref{g:nonlinear}, we get $\tp^* = g_\f(\p_0\,\cdot\, \tx_t)$ and thus \mbox{$p^* = \f(g_\f( \p_0\, \cdot\, \tx_t))$}.

We are now ready to bound the regret of the algorithm. Similar to Equation~\eqref{eq:step0}, we have 
\begin{align}
\E(R_t|\bcH_{t-1}) \le \frac{C}{2} (p^*_t - p_t)^2 &=\frac{C}{2} \Big[\f(g_\f(\p_0\cdot \tx_t))- \f(g_\f(\hp^k\cdot \tx_t))\Big]^2\nonumber\\
&\le\frac{C}{2} L (g_\f(\p_0\cdot \tx_t))- g_\f(\hp^k\cdot \tx_t))^2 \le \frac{L C}{2} |\tx_t\cdot (\p_0 - \hp^k)|^2\,,\label{eq:Rt4-2}
\end{align}
where $L \equiv \max_{|v|\le \f(M)} |\f'(v)|$ (since $\f$ is continuously differentiable, it attains a maximum over a bounded set.) In addition, we used the fact that $g'_\f(v) \le 1$ as explained below Equation~\eqref{g:nonlinear}. The inequalities above then follow from the mean-value theorem.

Given that $\tx_t$ is independent of ${\cH}_{t-1}$, we have
\begin{align}
\E(R_t|{\cH}_{t-1}) \le \frac{LC}{2}  \<\hp^k-\p_0, \tSigma_\phi (\hp^k-\p_0)\> \,, \label{eq:Rt4-00}
\end{align}
where $\tSigma_\phi = \E(\tx_t \tx_t^\sT)$. Using Assumption~\ref{ass2-2},
\begin{align}
\E(R_t) = \E\left(\E(R_t|\cH_{t-1})\right) \le \frac{1}{2}{LCC_{\max}} \E(\|\hp^k-\p_0\|_2^2)\,.  \label{eq:Rt4-000}
\end{align}

Rest of the proof is similar to proof of Theorem~\ref{proof:thm2} (see after Equation~\eqref{tot-regret}).

\subsection{Proof of Theorem~\ref{thm:regret2}}\label{proof:regret2}
We consider representation~\eqref{eq:model3} of the valuations and use the notation $\tx_t = (x_t,1)$, $\p_0 = (\ttheta_0,\tila_0)$. Fixe $k\ge 1$. Letting $x'_t = (-\tx_t,p_t)$, we can write the log-likelihood loss as:
\begin{eqnarray}
 \cL(\p') = - \frac{1}{k}\sum_{t \in \cA_k} &\bigg\{\ind(y_t =1) \log \left(1-F\left(\tx'_t \cdot (\p,\beta)\right)\right)
 + \ind(y_t =-1) \log \left(F\left(\tx'_t \cdot (\p,\beta)\right)\right) \bigg\}
\end{eqnarray}
Note that for $t\in \cA_k$, prices are posted uniformly at random in $[0,1]$ independently from the feature vector. Therefore, the population correlation works out at
\begin{align*}
\Sigma' \equiv \E(x'_t (x'_t)^{\sT}) = \begin{bmatrix}
\Sigma & 0 & 0\\
0&1&1/2\\
0&1/2&1/3
\end{bmatrix}
\end{align*}
Given that $\Sigma \succeq C_{\min} \id$, we have $\Sigma' \succeq C'_{\min} \id$, with $C'_{\min} \equiv \min(C_{\min},1/12)$.Therefore, the augmented feature vectors $x'_t$ satisfy Assumption~\ref{ass2}, with $C'_{\min}>0$. By applying
Proposition~\ref{thm:learning}, we get
\begin{align}\label{eq:learning2}
\|(\hp^k,\hbeta^k) - (\p_0,\beta_0)\|_2^2 \le C s_0 \lambda_k^2 /\ell_{\l1u}^2\,.
\end{align}
with probability at least $1 - 1/d - 2e^{-k/(c_0s_0)}$. We are now ready to bound the cumulative regret. Before proceeding, we need to figure out the clairvoyant policy.
\begin{lemma}\label{lem:opt-price}
Let $g$ be the pricing function corresponding to distribution $F = F_{0,1}$, given by $g(v) = v + \varphi^{-1}(-v)$, where $\varphi(v) = v - (1-F(v))/f(v)$ is the virtual valuation function.
Then, under model~\eqref{eq:model3}, the clairvoyant optimal prices are given by
\begin{align}
p^*_t = \frac{1}{\beta_0} g(\tx_t \cdot \mu_0)\,,
\end{align}
with $\mu_0 = (\ttheta_0,\tila_0)$ and $\tx_t = (x_t,1)$.
\end{lemma}
Proof of Lemma~\ref{lem:opt-price} is given in Appendix~\ref{app:opt-price}.

In the first period that the price is set randomly, we use the following naive bound on the regret:
\begin{align}\label{eq:priceB}
p_t^*  = \frac{1}{\beta_0} g(\tx_t\cdot \p_0) &\le \frac{2}{\beta_0}  |\tx_t\cdot \p_0|  \le  \frac{2}{\beta_0}  \|\tx_t\|_\infty \|\p_0\|_1 \le 2\l1u\,.
\end{align}
where in the first inequality we used the fact that $\varphi(v)$ is increasing for log-concave distribution and hence $g(v) = v+\varphi^{-1}(-v) \le v+|v| \le 2|v|$. The last step holds because $\p_0 / \beta_0 = (\theta_0,\alpha_0)$ and
$\|(\theta_0, \alpha_0)\|_1\le \l1u$.

We next bound the regret at other periods of the episode. Let $\bcH_t = \{x'_1,\dotsc, x'_t, x'_{t+1}, z_1, \dotsc, z_t\}$ be the history set up to time $t$. Similar to~\eqref{eq:step0}, we write
\begin{align}
\E(R_t | \bcH_{t-1}) \le \frac{C}{2} (p^*_t - p_t)^2 &= \frac{C}{2} \Big(\frac{1}{\beta_0} g(\tx_t\cdot \p_0) - \frac{1}{\hbeta^k} g(\tx_t\cdot \hp^k)\Big)^2\nonumber\\
& \le \frac{C}{\beta_0^2} \Big(g(\tx_t\cdot \p_0)  - g(\tx_t\cdot \hp^k) \Big)^2  + {C} \Big(\frac{1}{\beta_0} - \frac{1}{\hbeta^k} \Big)^2 g(\tx_t\cdot \hp^k)^2\nonumber\\
&\le   \frac{C}{\beta_0^2} |\tx_t\cdot (\p_0 - \hp^k)|^2 + \frac{4C \l1u^2}{\beta_0^2} (\hbeta^k - \beta_0)^2\,.
\end{align} 
In the last step, the first term is bounded using 1-Lipschitz property of $g$ and the second term is bounded using the observation $(1/\hbeta^k) g(\tx_t\cdot \hp^k) \le 2\l1u$, which can be derived similar to Equation~\eqref{eq:priceB}.
Recalling our notation $\cH_t = \bcH_t \backslash \{x'_{t+1}\}$ and applying the law of iterated expectations, we have
\begin{align}
\E(R_t | {\cH}_{t-1}) &\le  \frac{C}{\beta_0^2} C_{\max} \|\p_0 - \hp^k\|_2^2 + \frac{4C \l1u^2}{\beta_0^2}({\hbeta^k}- {\beta_0})^2 \nonumber\\
&\le C_1 \Big(\| \hp^k - \p_0 \|_2^2+ ({\hbeta^k}- {\beta_0})^2\Big)\nonumber \\
&= C_1 \| (\hp^k,\hbeta^k) - (\p_0,\beta_0) \|_2^2 \,,
\end{align} 
with $C_1 = \max(C C_{\max}/\beta_0^2, 4C_1 \l1u^2/\beta_0^2)$.

Therefore, by applying bound~\eqref{eq:learning2} and following similar lines as in proof of Theorem~\ref{thm:regret} (see Equation~\eqref{eq:Rt4} onwards), we bound the total regret during episode $k$ as follows:

Given that episode $k$ is of length $k$,
\begin{align}
\Reg (k{\rm th \,\,\,Episode})&\le C' \frac{s_0 k \lambda_k^2}{\ell_{\l1u}^2} 
\le 16C' \frac{u_{\l1u}^2}{\ell_{\l1u}^2} s_0 \log(d)  \,,\label{eq:EPSk}
\end{align}
for some constant $C>0$. Here, we use that fact that episode $k$ is of length $k$.

We next argue that the number of episodes before time $T$ is at most $K_0 = \sqrt{2 T_0}$. To see this, it suffices to note that the total number of
time periods after $K_0$ episodes is at least $K_0 + \sum_{c=1}^{K_0} c \ge K_0 (K_0+1)/2 \ge T$.

Therefore, by using bound~\eqref{eq:EPSk}, we get
\begin{align}
\Reg(T) \le \sum_{k=1}^{K_0} \Reg (k{\rm th \,\,\,Episode})  = O(s_0 (\log d) \sqrt{T}).
\end{align}


For the lower bound $\Omega(\sqrt{T})$, note that under model~\eqref{eq:model3} we can define the (scaled) customer's utility as
\begin{align}
\tilde{u}(x_t) = \ttheta_0\cdot x_t + \tila_0 - \beta_0 p +\tz_t\,.
\end{align}
Then, a purchase occurs if $\tilde{u}(x_t) >0$.  Following our discussion in Section~\ref{uninformative} (see after Equation~\eqref{eq:utility}), since $\beta_0$ is unknown, the uninformative prices do exist and therefore $\Omega(\sqrt{T})$ applies to this case.
\subsection{Proof of Theorem~\ref{thm:regret3}}\label{proof:thm-regret3}
We begin by stating a
bound on the mean squared error of the estimator $\hp^k$ given by optimization~\eqref{eq:ML4}. Without loss of generality, we can assume that the noise distribution is zero-mean. Otherwise, the mean can be absorbed in the model intercept $\alpha_0$. 

\begin{proposition}\label{thm:learning2}
Consider linear model~\eqref{eq:model} under Assumption~\ref{ass2}, where the noise term $z_t$ are generated from an unknown distribution with mean zero and support in $[-\delta,\delta]$. Also, suppose that $\p_0\in \Omega$ and let $\hp^k$ be the solution of optimization problem~\eqref{eq:ML4} with $K = \l1u +\delta$ and $\lambda \ge 8(K+\l1u) \sqrt{\dfrac{\log d}{ck}}$. Then, there exist positive constants $c_0$ and $C$ such that, the following inequality holds
 with probability at least $1- 1/d - 2e^{-ck/(c_0s_0)}$:
\begin{align}\label{L3B}
\|\hp^k - \p_0\|_2^2\le C {{s_0} \lambda^2}\,.
\end{align}
\end{proposition}
The proof of Proposition~\ref{thm:learning2} is given in Appendix~\ref{app-learning2}.

With Proposition~\ref{thm:regret3} in place, we next bound the regret of DIP policy. First, we show that the regret incurred during the exploration phase of episode $k$ is $O(1)$.  Since the noise is bounded, we have the following bound on the customer's valuation at each period
\begin{align}
v_t \le |\tx_t \cdot \p_0| + |z_t| \le \l1u + \delta = K\,.
\end{align} 
Therefore, the regret against a clairvoyant that can extract the valuation at each period is also bounded by $K$, and the regret using the exploration phase of episode $k$ is bounded by $c K$.

Next, we bound the regret incurred during the exploitation phase of episode $c$. During this phase, DIP policy offers prices $p_t = \hp^k\cdot\tx_t -2\delta$. The revenue generated can be lower bounded as follows:
\begin{align*}
p_t \ind (v_t\ge p_t) &= p_t \ind(\tx_t\cdot \p_0 +z_t \ge \tx_t\cdot \hp^k -2\delta)\\
&= p_t \ind(2\delta +z_t \ge \tx_t\cdot (\hp^k -\p_0))\\
&\ge p_t \ind(\delta  \ge | \tx_t\cdot (\hp^k -\p_0)|)\,,
\end{align*}
where we used the fact that $|z_t| \le \delta$. Consequently, the regret at each period of this phase can be bounded as follows:
\begin{align}
R_t &= v _ t - p_t \ind (v_t\ge p_t) \nonumber\\
&\le  v_t  - p_t \ind(\delta  \ge |\tx_t\cdot (\hp^k -\p_0)|) \nonumber\\
&= (v_t - p_t) \ind(\delta  \ge |\tx_t\cdot (\hp^k -\p_0)|) +  v_t \ind(\delta < |\tx_t\cdot (\hp^k -\p_0)|) \nonumber\\
& = \Big(z_t +\tx_t\cdot(\p_0 - \hp^k) + 2\delta\Big)\cdot \ind(\delta  \ge |\tx_t\cdot (\hp^k -\p_0)|) + K \ind(\delta < |\tx_t\cdot (\hp^k -\p_0)|) \nonumber\\
&\le 4\delta  + K \ind(\delta < |\tx_t\cdot (\hp^k -\p_0)|)\label{adversary-1}
\end{align}
 Furthermore, by Markov inequality, 
 \begin{align}
 \prob\left( |\tx_t\cdot (\hp^k -\p_0)| >\delta \right) \le \frac{1}{\delta^2}\, \E\left(|\tx_t\cdot (\hp^k -\p_0)|^2\right) 
 \le \frac{1}{\delta^2} C_{\max} \E(\|\hp^k -\p_0\|^2)\,,\label{adversary-2}
 \end{align} 
 where in the last step, we have first computed the expectation with respect to $\tx_t$ and used the fact that $\tx_t$ is independent from the residual $\hp^k -\p_0$. 
 
 Putting Equation~\eqref{adversary-1} and \eqref{adversary-2} together, we obtain
 \begin{align*}
 \E(R_t) &\le 4\delta +\frac{K}{\delta^2}  C_{\max} \E(\|\hp^k -\p_0\|^2) 
 \end{align*}
 Using the result of Proposition~\ref{thm:learning2} and following a similar argument as in the proof of Theorem~\ref{proof:thm2}, we bound the total regret incurred in episode $k$. Given that episode $k$ is of length $k$, we obtain
\begin{align}
\Reg (k{\rm th \,\,\,Exploitation\,\, Phase})&\le 4\delta k +\frac{K}{\delta^2}  C_{\max} C {s_0 k \lambda_k^2} \nonumber\\
&\le 4\delta k + 64(K+\l1u)^2 \frac{K}{c\delta^2}  C_{\max} C  s_0 \log d  \,,\label{eq:EPSk}
\end{align}
for some constant $C>0$. 

 
 Now, we are ready to bound the cumulative regret incurred in the first $T$ periods. Note that the way the cycles are defined in DIP policy, the number of cycles up to time $T$ is at most $\sqrt{2 T}$.  Hence,
 \begin{align}
 \Reg(T) = \sum_{k=1}^{\sqrt{2T}} \Reg (k{\rm th \,\,\,Episode}) &\le cK \sqrt{2T} + \sum_{k=1}^{\sqrt{2T}} \bigg\{4\delta k + \frac{C'}{\delta^2} s_0 \log d \Bigg\}\nonumber\\
& = c K \sqrt{2T} + 4\delta T + \frac{C'}{\delta^2} s_0 (\log d)  \sqrt{T}\,,
 \end{align}
 where $C' = 64(K+\l1u)^2 K C_{\max} C /c$.


\section{Conclusion}
In this work, we leverage tools from statistical learning to design a dynamic pricing policy for a setting wherein the products are described via high-dimensional features.
Our policy is computationally efficient and by exploiting the structure of demand parameters, it obtains a regret that scales gracefully with the features dimension and the time horizon. Namely, the regret of our algorithm scales linearly with the sparsity of the optimal solution and logarithmically with the dimension. We also show an $O(\log^2 T)$ dependence of the regret on the length of the horizon. On the flip side, we provide a lower-bound of $O(\log T)$ on the regret of any algorithm that does not know the true parameters of the model in advance. 

A natural next step is providing a tight bound on the regret, closing the gap between the derived upper and lower bounds. 
Another step would be assuming that $\th^*$ is not exactly sparse, but it can be well approximated by a sparse vector, i.e, $\|\tth -\th_{s_0}\|_1\le \delta$ for some $s_0$-sparse vector $\th_{s_0}$. An interesting question is to figure out how the regret scales with $\delta$. 

{The choice model that proposed in this work assumes one product arrived at each period, and describes the customer's purchase behavior based on the product features and the posted price. A more general 
choice model would be the one that assumes multiple products at each period. More specifically, each customer has a ``consideration" set which includes products left after the customer has narrowed down her choices based on her own personal screening criteria, and then chooses the product from this set which brings maximum utility. (We model the no purchase option as an extra product). This generalization is the focus of a future work.}

We also believe the ideas and techniques developed in this work can be be applied to other settings such as personalized pricing where information about the buyers can be used for price differentiation or optimizing reserve prices in online ad auctions.
Another application would be assortment optimization and learning consumer choice models both in terms of the role of the structure~\cite{farias2013nonparametric,Kallus2016assortment} as well as  personalization~\cite{golrezaei2014real,chen2015statistical} in data-rich environments.

\newpage
\appendix


\section{Proof of Proposition~\ref{thm:learning}}\label{app:RE}

We start by reviewing the notion of \emph{restricted eigenvalue} (RE) which is commonplace in high-dimensional statistical estimation.

\begin{definition}
For a given matrix $A\in \reals^{d\times d}$ and some integer $s$ such that $1\le s_0 \le d$ and a positive number $c$, 
we say that \emph{Restricted Eigenvalue} (RE) condition is met if 
$$\kappa^2(A,s_0,c) \equiv \min_{\substack{J\subseteq [p]\\ |J|\le s_0}}\, \min_{\substack{v\neq 0\\ \|v_{J^c}\|_1\le c \|v_J\|_1}} \frac{v^\sT A v}{\|v_J\|_2^2} > 0\,.$$
\end{definition}  
It is shown in~\cite{buhlmann2011statistics} and \cite{rudelson2013reconstruction} that when two matrices $A_0$, $A_1$ are close to each other (in the maximum 
element-wise norm) compared to sparsity $s_0$, the RE condition for $A_0$ implies the RE condition for $A_1$. This is particularly useful when $A_0$ is 
a population covariance matrix and $A_1$ is a corresponding empirical covariance matrix.  To apply this result to our case, let $\tX \in \reals^{n\times d}$ be the feature matrix with rows $\tx_t$, corresponding to 
$n$ products. Let $\tSigma = \E(\tx_t\tx_t^\sT)$. Given that $\E(x_t) = 0$, we have
\begin{align}
\tSigma = \begin{bmatrix}
\Sigma & 0\\
0& 1
\end{bmatrix}
\end{align}
Further, by Assumption~\ref{ass2}, we have $\Sigma \succeq C_{\min}\id$. Without loss of generality, we can assume $C_{\min} \le 1$, which implies $\tSigma \succeq C_{\min} \id$. Therefore, $\tSigma$ satisfies RE condition with $\kappa^2(\tSigma,s_0,3) \ge C_{\min}$.  By using the following result, we conclude that $\hSigma = (\tX^\sT \tX)/n$ also satisfies RE condition with $\kappa^2(\hSigma,s_0,3)\ge C_{\min}/2$.


\begin{proposition}\label{pro:comp}
Let $\hSigma = (\tX^\sT \tX)/n$ and let $S = \supp(\p_0)$ be the support of $\p_0$.
Under Assumption~\ref{ass2}, $\hSigma$ satisfies the restricted eigenvalue condition with constant $\kappa(\hSigma,s_0,3) \ge \sqrt{C_{\min}/2}$, 
with probability $1-e^{-2n/(c_0s_0)}$ and $c_0 = 768/C_{\min}^2$, provided that $n\ge c_0 s_0\log d$, 
\end{proposition}

Proposition~\ref{pro:comp} follows from the results established in~\cite{buhlmann2011statistics} and \cite{rudelson2013reconstruction}.
We outline the main steps of its proof in Appendix~\ref{proof:comp} for the reader's convenience.

By the second-order Taylor's theorem, expanding around $\p_0$ we have
\begin{align}\label{eq:taylor}
\cL(\p_0) - \cL(\hp) = -\<\nabla \cL(\p_0), \hp-\p_0\> -\frac{1}{2} \<\hp-\p_0,\nabla^2\cL(\tilp)(\hp-\p_0)\>\,,
\end{align}
for some $\tilp$ on the line segment between $\p_0$ and $\hp$. 
Invoking~\eqref{eq:loss}, we have
\begin{align}\label{eq:nabla-nabla2}
\nabla \cL(\p) = \frac{1}{n} \sum_{t=1}^n \xi_t(\p) \tx_t\,, \quad \nabla^2 \cL(\p) = \frac{1}{n} \sum_{t=1}^n \eta_t(\p) \tx_t \tx_t^\sT\,,
\end{align}
where $\nabla$ and $\nabla^2$ represents the gradient and the hessian w.r.t $\th$. Further,
\begin{eqnarray*}
\xi_t(\p) &=& -\frac{{f}(u_t(\p))}{F(u_t(\p))}\ind(y_t = -1) + \frac{{f}(u_t(\p))}{1-F(u_t(\p))} \ind(y_t = +1) \\
&=& -\log'F(u_t(\p)) \ind(y_t = -1) - \log'(1-F(u_t(\p))) \ind(y_t = +1)
\end{eqnarray*}
\begin{eqnarray*}
\eta_t(\p)  &=&  \bigg(\frac{f(u_t(\p))^2}{F(u_t(\p))^2} - \frac{{f'}(u_t(\p))}{F(u_t(\p))} \bigg)\ind(y_t = -1) +\bigg(\frac{f(u_t(\p))^2}{(1-F(u_t(\p)))^2}+ \frac{{f'}(u_t(\p))}{1-F(u_t(\p))} \bigg)\ind(y_t = +1)\\
&=& -\log''F(u_t(\p)) \ind(y_t = -1) - \log''(1-F(u_t(\p))) \ind(y_t = +1)\,,
\end{eqnarray*}
where $u_t(\p) = p_t-\<\tx_t,\p\>$, and $\log'F(x)$ and $\log''F(x)$ represent first and second derivative w.r.t $x$, respectively.

By Equation~\eqref{P}, we have
$$|u_t(\p_0)| \le |p_t| + \|\tx_t\|_\infty \|\p_0\|_1\le 3\l1u\,.$$
Further, recall that the sequences $\{p_t\}_{t=1}^n$ and $\{x_t\}_{t=1}^n$  are 
independent of $\{z_t\}_{t=1}^n$. Therefore, $\{u_t(\p_0)\}_{t=1}^T$ and $\{z_t(\p_0)\}_{t=1}^T$ are independent 
and by~\eqref{eq:prob-model}, we have $\E[\xi_t(\p_0)] = \E[\E[\xi_t(\p_0)|u_t(\p_0)]]=0$. 
Further, by definition of $u_{\l1u}$, cf. Equation~\eqref{eq:uM}, we have $|\xi_t(\p_0)|\le u_{\l1u}$.

We next introduce the set
\begin{align}\label{eq:F}
\cF \equiv \bigg\{\|\nabla \cL(\p_0)\|_\infty \le 2u_{\l1u} \sqrt{\frac{\log d}{n}} \bigg\}\,.
\end{align}
By applying Azuma-Hoeffding inequality followed by union bounding over $d$ coordinates of feature vectors, we obtain
$\prob(\cF) \ge 1- 1/d$.

On the other note, $\|\p_0\|_1, \|\hp\|_1 \le \l1u$ and hence $\|\tilp\|_1\le \l1u$. This implies that $|u_t(\tilp)|\le 3\l1u$. Therefore, by definition of 
$\ell_{\l1u}$, cf. Equation~\eqref{eq:lM}, we have $\eta_t(\tilp) \ge \ell_{\l1u}$. Recalling Equation~\eqref{eq:nabla-nabla2}, we get $\nabla^2 \cL(\tilp) \succeq \ell_{\l1u} (\tX^\sT \tX/n)$.

 
By optimality of $\hp$, we write
\begin{align}
\cL(\hp) +\lambda \|\hp\|_1 \le \cL(\p_0) +\lambda \|\p_0\|_1\,,
\end{align}
and by rearranging the terms and using~\eqref{eq:taylor}, we arrive at
\begin{align}\label{eq:OPT-S1}
 \frac{\ell_{\l1u}}{n}\|\tX(\p_0-\hp)\|^2 + \lambda\|\hp\|_1\le \|\nabla \cL(\p_0)\|_\infty \|\hp-\p_0\|_1+\lambda \|\p_0\|_1\,.
\end{align}
Form now on, the analysis is exactly similar to the oracle inequality for Lasso estimator. We bring the analysis here for the reader's convenience.

Choosing  $\lambda \ge 4u_{\l1u} \sqrt{(\log d)/n}$, we have on $\cF$
\begin{align}\label{eq:basic}
 \frac{2\ell_{\l1u}}{n}\|\tX(\p_0-\hp)\|^2 + 2 \lambda\|\hp\|_1\le \lambda \|\hp-\p_0\|_1+ 2\lambda \|\p_0\|_1\,.
\end{align}
Let $S = \supp(\p_0)$. On the left-hand side using triangle inequality, we have
$$\|\hp\|_1 =\|\hp_S\|_1 +\|\hp_{S^c}\|_1\ge \|\hp_S\|_1 -\|\hp_S-\p_{0,S}\|_1+\|\hp_{S^c}\|_1\,.$$
On the right-hand side, we have
$$ \|\hp-\p_0\|_1 = \|\hp_S-\p_{0,S}\|_1 + \|\hp_{S^c}\|_1\,.$$
Using these two inequalities in~\eqref{eq:basic}, we get
\begin{align}\label{eq:basic2}
\frac{2\ell_{\l1u}}{n}\|\tX(\p_0-\hp)\|^2 + \lambda\|\hp_{S^c}\|_1\le 3\lambda \|\hp_{S}-\p_{0,S}\|_1\,.
\end{align}

 We next write
\begin{align*}
\frac{2\ell_{\l1u}}{n}\|\tX(\p_0-\hp)\|^2 + \lambda \|\hp-\p_0\|_1
&= \frac{2\ell_{\l1u}}{n}\|\tX(\p_0-\hp)\|^2 + \lambda \|\hp_S-\p_{0,S}\|_1 + \lambda \|\hp_{S^c}\|_1\\
&\stackrel{(a)}{\le} 4 \lambda \|\hp_S-\p_{0,S}\|_1 \stackrel{(b)}{\le} 4\lambda \sqrt{s_0}\|\hp_S-\p_{0,S}\|_2\\
&\stackrel{(c)}{\le} \frac{4\lambda \sqrt{2s_0}}{\sqrt{n C_{\min}}}\|X(\hp-\p_0)\|_2\\
&\stackrel{(d)}{\le} \frac{\ell_{\l1u}}{n}\|\tX(\hp-\p_0)\|_2^2 +  \frac{8\lambda^2 s_0}{\ell_{\l1u} C_{\min}}\,,
\end{align*}
where $(a)$ follows from Equation~\eqref{eq:basic2}; $(b)$ holds by Cauchy-Shwarz inequality; $(c)$ follows form the RE condition, which holds for $\hSigma = (\tX^\sT \tX)/n$ as stated by Proposition~\ref{pro:comp}, with $\kappa(\hSigma,s_0,3)\ge \sqrt{C_{\min}/2}$, and recalling the inequality $\|\hp_{S^c}-\p_{0,S^c}\|_1 = \|\hp_{S^c}\|_1\le 3\|\hp_{S}- \p_{0,S}\|$ as per Equation~\eqref{eq:basic2}; Finally $(d)$ follows from the inequality $2\sqrt{ab} \le a^2+b^2$.
Rearranging the terms, we obtain
\begin{align}\label{eq:B1}
\frac{\ell_{\l1u}}{n}\|\tX(\p_0-\hp)\|^2 + \lambda \|\hp-\p_0\|_1 \le \frac{8\lambda^2 s_0}{\ell_{\l1u} C_{\min}}\,.
\end{align}
Applying the RE condition again to the L.H.S of~\eqref{eq:B1}, we get  
\begin{align}\label{eq:B2}
C_{\min} \frac{\ell_{\l1u}}{2} \|\p_0-\hp\|_2^2\le \frac{\ell_{\l1u}}{n}\|\tX(\p_0-\hp)\|^2  \le \frac{8\lambda^2 s_0}{\ell_{\l1u} C_{\min}}\,,
\end{align}
and therefore,
\begin{align}\label{eq:B3}
\|\p_0-\hp\|_2^2 \le \frac{16 s_0 \lambda^2}{\ell_{\l1u}^2 C_{\min}^2} \,.
\end{align}
The result follows.

\subsection{Proof of Proposition~\ref{pro:comp}}
\label{proof:comp}
The proof follows by combining two lemmas from~\cite{buhlmann2011statistics}.

We show the desired result holds for a more general case, namely for $\tX$ with subgaussian entries.
Before stating the proof, we recall a few definitions and notations.  
\begin{definition}
A random variable $\nu$ is subgaussian if there exist constants $L,\sigma_0$ such that
$$\E(e^{\frac{\nu^2}{L^2}})\le \frac{\sigma_0^2}{L^2}+1\,.$$
\end{definition}

Note that bounded random variables are subgaussian. Specifically, if $|\nu|\le \nu_{\max}$, then $\nu$ is subgaussian
with $L = \nu_{\max}$ and $\sigma_0 = \nu_{\max} \sqrt{e-1}$.

For a matrix $A$, we let $\|A\|_\infty$ denote its (element wise) maximum norm, i.e., $\|A\|_\infty = \max_{i,j} |A_{ij}|$.
The next lemma shows that if two matrices are close enough in maximum norm and if the compatibility condition holds for one of them then it would also hold for the other one. 
\begin{lemma}\label{lem:aux1}
Suppose that the restricted eigenvalue (RE) condition holds for $\Sigma_0$ with constant $\kappa(\Sigma_0,s_0,3)>0$.
If 
$$\|\Sigma_0 -\Sigma_1\|_\infty \le \frac{\kappa^2(\Sigma_0,s_0,3)}{32s_0}\,,$$
then the RE condition holds for $\Sigma_1$ with constant $\kappa(\Sigma_1,s_0,3)\ge \kappa(\Sigma_0,s_0,3)/\sqrt{2}$.
\end{lemma}
\begin{proof}{Proof of Lemma~\ref{lem:aux1}}
We refer to Problem 6.10 of~\cite{buhlmann2011statistics}. 
\end{proof}
\begin{lemma}\label{lem:aux2}
Consider $\tX\in \reals^{n\times p}$ with i.i.d. rows generated from a distribution with covariance $\tSigma\in \reals^{d \times d}$. 
Let $\hSigma = \tX^\sT \tX/n$ be the corresponding empirical covariance.  Further, suppose that the entries of $X$ are
uniformly subgaussian with parameters $L,\sigma_0$. 
If $n\ge c_0 L s_0\log d$ with $c_0 = 768 L/\kappa^2(\tSigma,s_0,3)$, then
\begin{align}
\prob\bigg[\|\hSigma - \tSigma\|_\infty \ge \frac{\kappa^2(\tSigma,s_0,3)}{384 s_0} \Big(2+\frac{7\sigma_0}{L} \Big) \bigg] \le e^{-\frac{2n}{cs_0}}\,.
\end{align}
\end{lemma}
\begin{proof}{Proof of Lemma~\ref{lem:aux2}}
The result follows readily from Problem 14.3 on page 535 of~\cite{buhlmann2011statistics}. 
\end{proof}

Next we note that $\tSigma$ satisfies the restricted eigenvalue condition with constant $\kappa^2(\tSigma,s_0,3) \ge {C_{\min}}$ because of Assumption~\ref{ass2}.
Further, since $\|\tx_t\|_\infty \le 1$, we can apply the result of Lemma~\ref{lem:aux2} with $L =1$, $\sigma_0 = \sqrt{e-1}$. Proposition~\ref{pro:comp} then follows from 
Lemma~\ref{lem:aux1}.

\section{Proof of Proposition~\ref{pro:E2}}\label{proof:pro-E2}
Define the event $\cB_n$ as follows:
\begin{align}
\cB_n \equiv \Big\{ \tX\in \reals^{n\times d}:\, \sigma_{\min} (\tX^\sT \tX/n) > C_{\min}/2  \Big\}\,.
\end{align}

Using concentration bounds on the spectrum of random matrices with subgaussian rows (see~\cite[Equation (5.26)]{vershynin2010introduction}), there exist constants $c, c_1>0$
such that for $n>c_1 d$, we have $\prob(\cB_n) \ge 1 - e^{-cn^2}$.
 
For $\gamma>0$, we define the event $\cF_\gamma = \{\|\nabla \cL(\p_0)\|_\infty \le \gamma\}$. Using characterization~\eqref{eq:nabla-nabla2},
and by applying Azuma-Hoeffding inequality (similar to our argument after Equation~\eqref{eq:F}), we obtain 
\begin{align}
\prob(\cF_\gamma) \ge 1 - d \exp\Big(-\frac{n\gamma^2}{2 u_{\l1u}^2} \Big)\,. \label{eq:FgammaB}
\end{align}
We also let $\cE_{1,n}\equiv \cB_n \cap \cF_{\lambda/2}$, $\cE_{2,n}\equiv \cB_n\cap \cF_{\lambda/2}^c$. To lighten the notation, we use the shorthand $D \equiv \|\hp-\p_0\|_2^2$.

We then have
\begin{align}
\E(D) = \E(D\cdot \ind(\cB_n^c)) + \E(D\cdot \ind(\cE_{1,n})) +  \E(D\cdot \ind(\cE_{2,n}))\,.
\end{align}

We treat each of the terms on the right-hand side separately. 
\begin{itemize}
\item {\bf Term 1:} We have
\begin{align}
\E(D\cdot \ind(\cB_n^c)) \le 4\l1u^2 \prob(\cB_n^c) \le 4\l1u^2 e^{-cn^2}\,.
\end{align}
\item {\bf Term 2:} Similar to proof of Proposition~\ref{thm:learning}, on $\cE_{1,n}$, we have $D \le 16 s_0 \lambda^2/(\ell_{\l1u}^2 C_{\min}^2)$. Hence,
\begin{align}
\E(D\cdot \ind(\cE_{1,n})) \le \frac{16s_0 \lambda^2}{\ell_{\l1u}^2 C_{\min}^2} \,.
\end{align}
\item {\bf Term 3:} To bound term 3, we first prove the following lemma.
\begin{lemma}\label{lem:term3}
On event $\cE_n(\gamma)\equiv \cB_n \cap \cF_\gamma$, with $\gamma >\lambda/2$, we have 
\[D\le \Big(\frac{36}{C_{\min}^2 \ell_{\l1u}^2}\Big) \gamma^2 d.\]
\end{lemma}
Lemma~\ref{lem:term3} is proved in Section~\ref{proof:lem-term3}.
\end{itemize}

We next bound term 3 as follows. Let $L = 9\lambda^2 d/(C_{\min}^2 \ell_{\l1u}^2)$.
\begin{align}
\E(D\cdot \ind(\cE_{2,n})) &= \int_{0}^\infty \prob\Big(D\cdot \ind(\cE_{2,n}) >  \alpha \Big) \de \alpha \nonumber\\
& = L  \int_{0}^\infty \prob\Big(D\cdot \ind(\cE_{2,n}) > L c\Big) \de c \nonumber\\
& = L  \int_{0}^1 \prob\Big(D\cdot \ind(\cE_{2,n}) > L c\Big) \de c + L  \int_{1}^\infty \prob\Big(D\cdot \ind(\cE_{2,n}) > L c\Big) \de c\,.\label{eq:Term3}
\end{align}
For the first term on the right-hand side we write
\begin{align}\label{eq:Prob-t1}
L  \int_{0}^1 \prob\Big(D\cdot \ind(\cE_{2,n}) > L c\Big) \de c \le L\int_{0}^1 \prob(\cE_{2,n}) \le L \prob(\cF_{\lambda/2}^c) \le  \frac{L}{d}\,,
\end{align}
where the last step holds from Equation~\eqref{eq:FgammaB} with $\gamma = \lambda/2$.

We next upper bound the second term. For arbitrary fixed $c>1$, let $\gamma = \sqrt{Lc/(d\kappa)}$, with $\kappa = 36/(\ell_{\l1u}^2 C_{\min}^2)$.
It is easy to verify that $\gamma  = \sqrt{c} \lambda/2 > \lambda/2$. Further, by virtue of Lemma~\ref{lem:term3}, on $\cE(\gamma)$ we have $D\le \kappa \gamma^2 d  = Lc$.
Hence, 
\begin{align}\label{eq:Pt2-1}
\prob\Big(D\cdot \ind(\cE_{2,n}) > L c\Big) \le \prob\Big(\cE(\gamma)^c \cap \cE_{2,n}\Big) \le \prob(\cF_{\gamma}^c \cap \cB_n) \le \prob(\cF_{\gamma}^c)
\end{align} 
Further, applying Equation~\eqref{eq:FgammaB} and plugging for $\gamma$, we obtain
\begin{align}\label{eq:cF-0}
\prob(\cF_{\gamma}^c)\le d\exp\Big(-\frac{nLc}{2u_{\l1u}^2 \kappa d}\Big) = d\exp\Big(-\frac{n\lambda^2}{8u_{\l1u}^2}\Big) \le d^{1-2c}\,.
\end{align}
Here, the second second step follows from definition of $L$ and the last step holds because $\lambda \ge 4u_{\l1u} \sqrt{(\log d)/n}$.

Combining Equations~\eqref{eq:Pt2-1} and \eqref{eq:cF-0}, we have
\begin{align}\label{eq:Pt2-2}
L  \int_{1}^\infty \prob\Big(D\cdot \ind(\cE_{2,n}) > L c\Big) \de c \le L\int_{1}^{\infty} d^{1-2c}\, \de c\le \frac{L}{2d \log d}\,.
\end{align}

Using bounds~\eqref{eq:Prob-t1} and \eqref{eq:Pt2-2} in Equation~\eqref{eq:Term3}, we obtain
 \begin{align}
\E(D\cdot \ind(\cE_{2,n})) \le \frac{L}{d} \left(1+ \frac{1}{2\log d} \right)\le \frac{3L}{2d} < \frac{16\lambda^2}{\ell_{\l1u}^2 C_{\min}^2}\,.
\end{align}

The result follows by putting the upper bounds on the three terms together.
\subsection{Proof of Lemma~\ref{lem:term3}}\label{proof:lem-term3}
We start by rewriting Equation~\eqref{eq:OPT-S1}, which follows from optimality of $\htheta$ and log-concave property of the loss function.
\begin{align}\label{eq:OPT-S2}
 \frac{\ell_{\l1u}}{n}\|\tX(\p_0-\hp)\|^2 + \lambda\|\hp\|_1\le \|\nabla \cL(\p_0)\|_\infty \|\hp-\p_0\|_1+\lambda \|\p_0\|_1\,.
\end{align}
On event $\cE_n(\gamma)$, Equation~\eqref{eq:OPT-S2} implies that
\begin{align}
\frac{1}{2} C_{\min} \ell_{\l1u} \|\p_0-\hp\|_2^2 + \lambda\|\hp\|_1\le \gamma \|\hp-\p_0\|_1+\lambda \|\p_0\|_1\,.
\end{align}
Using the assumption $\gamma>\lambda/2$ and our shorthand $D \equiv \|\hp-\p_0\|_2^2$, we get
\begin{align*}
\frac{1}{2} C_{\min} \ell_{\l1u} \|\hp-\p_0\|_2^2 &\le \gamma \|\hp-\p_0\|_1+\lambda \|\p_0\|_1 - \lambda\|\hp\|_1\\
&\le (\gamma+\lambda) \|\hp-\p_0\|_1\\
&\le 3 \gamma \|\hp-\p_0\|_1 \le 3 \gamma \sqrt{d} \|\hp-\p_0\|_2\,.
\end{align*}
Writing the above bound in terms of our shorthand $D\equiv \|\hp-\p_0\|_2^2$, we obtain the desired result.
\section{Proof of Technical Lemmas}\label{sec:proofs}

\subsection{Proof of Lemma~\ref{lem:phi-monotone}}\label{app:phi-monotone}
We write the virtual valuation function as $\varphi(v) = v -1/\lambda(v)$ where $\lambda(v) = \frac{f(v)}{1-F(v)} = -\log'(1-F(v))$ is the hazard rate function. Since $1-F$ is log-concave, the hazard function $\lambda(v)$ is increasing which implies that
$\varphi$ is strictly increasing. Indeed, by this argument $\varphi'(v)>1$.
\subsection{Proof of Lemma~\ref{lem:g-Lipschitz}}\label{app:g-Lipschitz}
Recalling the definition $g(v) = v+\varphi^{-1}(-v)$, we have $g'(v) = 1-1/\varphi'(\varphi^{-1}(-v))$. Since $\varphi$ is strictly increasing by Lemma~\ref{lem:phi-monotone},
we have $g'(v)<1$. The claim $g'(v)>0$ follows if we show $\varphi'(\varphi^{-1}(-v)) > 1$. For this we refer to the proof of Lemma~\ref{lem:phi-monotone}, where we showed 
that $\varphi'(v)>1$ for all $v$.
\subsection{Proof of Lemma~\ref{lem:hessian}}\label{proof:hessian}
Let $\phi(v)$ and $\Phi(v)$ respectively denote the density and the distribution function of standard normal variable.
Function $h_t$ and its derivatives read as 
\begin{align}
r_t(p) &= p\Big(1-\Phi\Big(\frac{p-x_t\cdot \tth}{\sigma}\Big)\Big)\,,\\
r_t'(p)&= 1-\Phi\Big(\frac{p-x_t\cdot \tth}{\sigma}\Big) - \frac{p}{\sigma}\phi\Big(\frac{p-x_t\cdot \tth}{\sigma}\Big)\,,\\
r_t''(p)&= \frac{1}{\sigma} \Big[ \frac{p}{\sigma} \Big(\frac{p-x_t\cdot \tth}{\sigma} \Big) - 2\Big] \phi\Big(\frac{p-x_t\cdot \tth}{\sigma}\Big)\,.
\end{align}
Define $\xi \equiv p^*_t - x_t\cdot \tth = g(x_t\cdot \tth) - x_t\cdot \tth$. Writing $r_t''(p^*_t)$ in term of $\xi$, we obtain
\begin{align}
r''_t(p^*_t) = \frac{1}{\sigma} \Big[\frac{\xi}{\sigma} \Big(\frac{1-\Phi(\xi/\sigma)}{\phi(\xi/\sigma)}\Big) - 2 \Big] \phi(\xi/\sigma)
\end{align}
By tail bound inequality for Gaussian distribution $1-\Phi(\xi/\sigma) \le (\sigma/\xi) \phi(\xi/\sigma)$ for $\xi\ge 0$. Therefore,
\begin{align}\label{eq:t1}
\frac{\xi}{\sigma} \Big(\frac{1-\Phi(\xi/\sigma)}{\phi(\xi/\sigma)}\Big) - 2 \le -1\,,
\end{align}
and the same bound obviously holds for $\xi<0$.

By definition of function $g$, $|\xi| \le 3\l1u$ with $\varphi$ being the virtual valuation fusion corresponding to the
Gaussian distribution.
Hence, $\phi(\xi/\sigma) \ge \phi(3\l1u/\sigma)$. Putting this together with~\eqref{eq:t1}, we get $r_t''(p^*_t) \le -c_1$ with 
$c_1 =(1/\sigma) \phi(3\l1u/\sigma)$.

For the second part of the Lemma statement, set $\delta \le \min\Big\{3\l1u,{\sigma^2}{/(18\l1u)},\sigma^2\phi(3\l1u/\sigma)\Big\}$. 
For $p\in [p^*_t-\delta,p^*_t+\delta]$, we have
\begin{align}
\Big|\frac{p}{\sigma} \Big(\frac{p-x_t\cdot \tth}{\sigma} \Big) - \frac{p^*_t}{\sigma} \Big(\frac{p^*_t-x_t\cdot \tth}{\sigma} \Big)\Big|
\le \frac{1}{\sigma^2} {|p-p^*_t|} \cdot {|p+p^*_t-x_t\cdot \tth|}\le \frac{1}{\sigma^2} \delta (5\l1u+\delta)\le \frac{1}{2}\,.\label{eq:t2}
\end{align}
Using Equation~\eqref{eq:t1} we get $p(p-x_t\cdot \tth)/\sigma^2\le-1/2$.
%
%
Further,
\begin{align}\label{eq:t3}
\Big| \phi\Big(\frac{p-x_t\cdot \tth}{\sigma}\Big) - \phi\Big(\frac{p_t^*-x_t\cdot \tth}{\sigma}\Big) \Big| \le \frac{|p-p^*_t|}{2\sigma^2} \le \frac{\delta}{2\sigma^2} 
\le \frac{1}{2}\phi(3\l1u/\sigma)\,.
\end{align}
Therefore,
\begin{align}\label{eq:t4}
 \phi\Big(\frac{p-x_t\cdot \tth}{\sigma}\Big) \ge \phi(\xi/\sigma) - \frac{1}{2} \phi(3\l1u/\sigma) \ge \frac{1}{2}\phi(3\l1u/\sigma)\,.
\end{align}
Combining~\eqref{eq:t2}, \eqref{eq:t4} we obtain
\begin{align}\label{eq:h"p}
r_t''(p) &= \frac{1}{\sigma} \Big[ \frac{p}{\sigma} \Big(\frac{p-x_t\cdot \tth}{\sigma} \Big) - 2\Big] \phi\Big(\frac{p-x_t\cdot \tth}{\sigma}\Big)
\le -\frac{1}{4\sigma} \phi(3\l1u/\sigma)=  -\frac{c_1}{4}\,. 
\end{align}
The result follows.

\subsection{Proof of Lemma~\ref{lem:Emin}}\label{proof:Emin}
Let $Z = x\cdot v$ and $\tZ = Z/\|v\|_2$. Note that $\Var(\tZ) = 1$. 
Write the expectation in terms of the tail probability 
\begin{align}\label{eq:ECDF}
\E(\min(Z^2,\delta^2)) = \int_{0}^{\delta^2} \prob(Z^2\ge t)\de t = 
 \int_{0}^{\delta^2} \prob\Big(|\tZ|\ge \frac{\sqrt{t}}{\|v\|_2}\Big)\de t\,.
\end{align}
We consider two cases:
\begin{itemize}
\item $\delta \le\|v\|_2$: The right-hand side in~\eqref{eq:ECDF} can be lower bounded as
\begin{align}\label{eq:int0}
 \int_{0}^{\delta^2} \prob\Big(|\tZ|\ge \frac{\sqrt{t}}{\|v\|_2}\Big)\de t \ge
 \int_{0}^{\delta^2} \prob\Big(|\tZ|\ge \frac{\sqrt{t}}{\delta}\Big)\de t
 =  2\delta^2 \int_{0}^{1} t\, \prob(|\tZ|\ge t) \de t 
\end{align}
In the sequel, we provide two separate lower bounds for the right-hand side.

Let $\xi \equiv \prob(|\tZ|\ge 1)$. We have 
\begin{align}\label{eq:int1}
\int_{0}^{1} t\, \prob\Big(|\tZ|\ge t\Big) \de t \ge \int_{0}^{1} t\, \xi \de t \ge \frac{\xi}{2}\,. 
\end{align}
We proceed to obtain another bound which utilizes the fact $\Var(\tZ) = 1$.
\begin{align}\label{eq:int2}
\int_{0}^{1} t\, \prob(|\tZ|\ge t) \de t  &=  \int_{0}^{\infty} t\, \prob(|\tZ|\ge t) \de t  - \int_{1}^{\infty} t\, \prob(|\tZ|\ge t) \de t \nonumber\\
&= \frac{1}{2}\Var(\tZ) -   \int_{1}^{\infty} t\, \prob(|\tZ|\ge t) \de t
\end{align}
For $t\ge 1$, we have $ \prob(|\tZ|\ge t)\le \xi$. Further, by applying Chernoff bound, we get
\begin{align*}
\prob(|\tZ|\ge t) &= 2\prob(\tZ\ge t) = 2\prob(e^{\lambda \tZ}\ge e^{\lambda t}) \le e^{-\lambda t}{\E(e^{\lambda \tZ})}\\
&= 2e^{-\lambda t} \prod_{i=1}^d \bigg(\frac{e^{\lambda \frac{v_i}{\|v\|_2}}+e^{-\lambda \frac{v_i}{\|v\|_2}}}{2}\bigg) \le 2e^{-\lambda t} \prod_{i=1}^d e^{\lambda^2 \frac{v_i^2}{2\|v\|_2^2}} = 2e^{\frac{\lambda^2}{2} - \lambda t}\,.
\end{align*}
Setting $\lambda = t$ leads to 
$\prob(|\tZ|\ge t) \le 2e^{-\frac{t^2}{2}}$. Combining these bounds into~\eqref{eq:int1}, we obtain
\begin{align}\label{eq:int3}
\int_{0}^{1} t\, \prob(|\tZ|\ge t) \de t \ge \frac{1}{2} - \int_{1}^\infty t\, \min (2e^{-\frac{t^2}{2}}, \xi) \de t = \frac{1-\xi}{2} - \xi \log \Big(\frac{2}{\xi}\Big)
\end{align}
We summarize bounds~\eqref{eq:int1} and~\eqref{eq:int3} as in
\begin{align}\label{eq:int4}
\int_{0}^{1} t\, \prob\Big(|\tZ|\ge t\Big) \de t \ge \min_{\xi\in [0,1]}\max\left(\xi, \frac{1-\xi}{2} - \xi \log \Big(\frac{2}{\xi}\Big)\right) > 0.05
\end{align}
Turning back to Equation~\eqref{eq:int0}, in this case we have
\begin{align}\label{eq:int}
 \int_{0}^{\delta^2} \prob\Big(|\tZ|\ge \frac{\sqrt{t}}{\|v\|_2}\Big)\de t \ge
0.1\delta^2
\end{align}

\item $\delta\ge \|v\|_2$: Similar to the previous case, the right-hand side in~\eqref{eq:ECDF} can be lower bounded as
\begin{align}
 \int_{0}^{\delta^2} \prob\Big(|\tZ|\ge \frac{\sqrt{t}}{\|v\|_2}\Big)\de t \ge
 \int_{0}^{\|v\|_2^2} \prob\Big(|\tZ|\ge \frac{\sqrt{t}}{\|v\|_2}\Big)\de t
 \ge 0.1{\|v\|_2^2}
\end{align}
\end{itemize}
The above two cases can be summarized as $\E(\min(Z^2,\delta^2))\ge 0.1 \min(\|v\|_2^2,\delta^2)$.


\subsection{Proof of Lemma~\ref{lem:Fano}}\label{App:Fano}
We use a standard argument that relates minimax $\ell_2$-risk in terms of the error in  multi-way hypothesis
testing problem; See e.g.~\cite{yang1999information,AFLC}. Let $\{\tilth_1, \dotsc, \tilth_m\}$ be a $\delta$-packing of set $\Omega$, meaning that
their pairwise distances are all at least $\delta$. Parameter $\delta$ is free for now and its value will be determined later in the proof. 
We further let $P_j$ denote the induced probability on market values $(v(x_1), \dotsc,v(x_T))$, conditional
on $(x_1, \dotsc, x_T)$ and for $\tth = \tilth_j$. In other words, in defining distributions $P_j$ we treat feature vectors fixed. Let $\nu$ be random variable uniformly distributed on the hypothesis set 
$\{1,2,\dotsc, m\}$ which indicates the index of the true parameter, i.e, $\nu = j$ means $\tth = \tilth_j$. 

Define $\dis(\th_1^T,\th)\equiv \sum_{t=1}^T \min(\|\th_t - \th\|^2_2,C)$ and let $\mu$ be the value of $j$ for which $\dis(\th_1^T,\tilth_j)$ is a minimum.  
Suppose that $\delta$ is chosen such that $\delta^2\le C$. If $\dis(\th_1^T,\tilth_j) <\delta^2 T/4$ then $\mu =j$, because assuming otherwise, we have $\mu = j'\ne j$, and by triangle inequality
\begin{align}
\min(\|\tilth_{j'}-\tilth_j\|_2^2,C) &\le \min(2\|\th_t - \tilth_j\|^2_2 + 2\|\th_t - \tilth_{j'}\|^2_2, C)\nonumber \\
&\le \min(2\|\th_t - \tilth_j\|^2_2, C) +\min(2\|\th_t - \tilth_{j'}\|^2_2, C)\,,
\end{align}
for all $t$, where we used the inequality $\min(a+b,c) \le \min(a,c) + \min(b,c)$ for $a,b,c \ge 0$.
Summing over $t=1, 2, \dotsc, T$, we get 
$$T\min(\|\tilth_{j'}-\tilth_j\|_2^2,C)\le 2\dis(\th_1^T,\tilth_{j}) + 2\dis(\th_1^T,\tilth_{j'})\le 4 \dis(\th_1^T,\tilth_{j}) < \delta^2 T\,,$$
where we used the assumption $\mu = j'$. But this is a contradiction because $\|\tilth_{j'}-\tilth_j\|_2\ge \delta$ (they form a $\delta$-packing of $\Omega$) and $\delta^2\le C$.

Using Markov inequality, we can write
\begin{align}
\max_{j}\, \E_{P_j} \dis(\th_1^T,\tilth_j) &\ge \frac{\delta^2 T}{4} \max_j\, \prob\Big(\dis(\th_1^T,\tilth_j) \ge\frac{\delta^2 T}{4}\Big|\nu=j\Big)\nonumber\\
&\ge \frac{\delta^2 T}{4m} \sum_{j=1}^m \prob(\mu\neq j|\nu =j ) = \frac{\delta^2 T}{4} \prob(\mu\neq \nu)\,.\label{eq:Fano1}
\end{align}
We use Fano's inequality to lower bound the error probability on the right-hand side. We first construct a $\delta$-packing of $\Omega$ similar to the one proposed in~\cite[proof of Theorem 1]{raskutti2011minimax}.  

Let $s = s_0/2 \le d/2$ and define
$$\cA = \{(q,0)\in \reals^{d+1}:\, q\in \{-1,0,1\}^d :\,\; \|q\|_0 = s\}\,.$$ 
As proved in~\cite[Lemma 5]{raskutti2011minimax}, there exists a subset $\tilde{\cA}\subseteq \cA$ of cardinality $|\tilde{\cA|}\ge \exp(\frac{s}{2}\log \frac{d-s/2}{s})$ such that the Hamming distance between any two elements in $\tilde{\cA}$ is at least $s/2$. 
Next, consider the set $\sqrt{\frac{2}{s}}\delta \tilde{\cA}$ for some $\delta\le \l1u/\sqrt{2s}$. whose exact value to be determined later. Then, for $q$ in this set, $\|q\|_1= \sqrt{2s} \delta \le \l1u$ and hence $\sqrt{\frac{2}{s}}\delta \tilde{\cA} \subseteq \Omega_0$. Further, for $q,q' \in \sqrt{\frac{2}{s}}\delta \tilde{\cA}$, we have the following bounds:
\begin{align}
&\|q-q'\|_2^2\ge \delta^2\,,\label{eq:l2-q1}\\
&\|q-q'\|_2^2\le 8\delta^2\,. \label{eq:l2-q2}
\end{align}
By~\eqref{eq:l2-q1}, the set $\sqrt{\frac{2}{s}}\delta \tilde{\cA}$ forms a $\delta$-packing for $\Omega_0$ with size $|\tilde{\cA}|$.

We now turn back to bound~\eqref{eq:Fano1}.
Left-hand side can be lower bounded using Fano's inequality. We omit the details here as it is a standard argument and instead we refer to~\cite[proof of Theorem 1]{raskutti2011minimax} for details. Using Fano's inequality and bound~\eqref{eq:l2-q2}, we get
\begin{align}
\prob(\mu\neq \nu) = 1- \frac{\frac{8T}{2\sigma^2}\delta^2+\log (2)}{\frac{s}{2}\log (\frac{d-s/2}{s})}\,.
\end{align}
Choosing $\delta^2 \le \delta^2_1\equiv \frac{\sigma^2 s}{32T} \log (\frac{d-s/2}{s})$, we obtain $\prob(\mu\neq \nu) \ge 1/4$.
Therefore, setting $\delta^2 = \min(\frac{\l1u^2}{{2s}}, \delta^2_1,C)$ and combining with bound~\eqref{eq:Fano1}, we conclude that
\begin{align}
\min_{\th_1^T}\, \max_{\tth\in \Omega_0} \,\E( \dis(\th_1^T,\th_0)) \ge \frac{\delta^2 T}{16} =  
\frac{1}{16}\min\bigg\{\frac{\l1u^2 T}{2s}, \frac{\sigma^2 s}{32} \log \Big(\frac{d-s/2}{s}\Big),{CT}\bigg\}\,. \label{eq:l2risk-3}
\end{align}
Now since $s = s_0/2\le d/2$, we have $\log ((d-s/2)/s) \ge c \log (d/s)$ with some constant $c>0$. Therefore, by using Equation~\eqref{eq:l2risk-3} and substituting for $s = s_0/2$, we obtain
\begin{align}
\min_{\th_1^T}\, \max_{\tth\in \Omega_0} \,\E( \dis(\th_1^T,\th_0)) \ge L_1\,,
\end{align}
with 
\begin{align}
L_1 \equiv \frac{1}{16}\min\bigg\{\frac{\l1u^2 T}{s_0}, \frac{c\sigma^2 s_0}{64} \log \Big(\frac{d}{s_0}\Big),{CT}\bigg\}\,.\label{eq:L1}
\end{align}

We next derive another separate lower bound for minimax risk, by assuming that an oracle gives us the true support of $\tth$.
In this case, the least square estimator, applied to the observed features restricted to the true support $S$, achieves the optimal minimax $\ell_2$ rate. 
This implies that $\|\th_t-\tth\|_2^2 \ge c\sigma^2 s_0/t$, for  $t\ge s_0$ and a constant $c>0$. Therefore,
\begin{align}
\min_{\th_1^T}\, \max_{\tth\in \Omega_0} \,\E( \dis(\th_1^T,\th_0)) \ge 
\sum_{t=1}^T \min\Big(c\sigma^2\frac{s_0 }{t},C\Big)\,, 
\end{align}
from which we obtain
\begin{align}
\min_{\th_1^T}\, \max_{\tth\in \Omega_0} \,\E( \dis(\th_1^T,\th_0)) \ge  L_2\equiv  c' s_0\, \log(T/s_0)\,, \label{eq:l2risk-4}
\end{align}
for some constant $c'>0$, depending on $\sigma$ and $C$.

Combining bounds in~\eqref{eq:L1} and \eqref{eq:l2risk-4}, we have
\begin{align}
\min_{\th_1^T}\, \max_{\tth\in \Omega_0} \,\E( \dis(\th_1^T,\th_0)) &\ge \frac{1}{2} (L_1+L_2)\nonumber\\
&\ge \tC \bigg\{ s_0 \log \Big(\frac{T}{s_0}\Big)+ \min\bigg[\frac{T}{s_0},  s_0 \log \Big(\frac{d}{s_0}\Big)\bigg]\bigg\} \,, \label{eq:l2risk-5}
\end{align}
for a constant $\tC$ that depends on $C, \sigma, \l1u$.
The proof is complete.
%
%
%
%

\subsection{Proof of Lemma~\ref{lem:ass-equiv}}\label{app:ass-equiv} 
We recall the notion of \emph{$0$-property} established by~\cite{ponomarev1987submersions}.
\begin{definition}
A continuous function has the \emph{$0$-property}, if the pre-image of any set of probability zero is a set of probability zero.
\end{definition}
As proved in~\cite[Theorem 1]{ponomarev1987submersions}, if a function $\phi:\cX\subseteq \reals^d \mapsto \reals^d$ is continuously differentiable, 
then it satisfies $0$-property if and only if its derivative $\cD\phi$ is full rank for almost all $x\in\cX$. Therefore, we need to show that under Assumption~\ref{ass1}, if $\phi$ has $0$-property, then Assumption~\ref{ass2-2} holds true. 

Supposing otherwise, there exists a nonzero $v\in \reals^d$ such that $v^\sT \Sigma_{\phi} v = 0$. Therefore, $\E((z\cdot \phi(x))^2) = 0$ which implies that $z\cdot \phi(x) = 0$, almost surely. Define $S\equiv \{z\in \reals^d:\, z\cdot \phi(x) = 0\}$.  Space $S$ is $(d-1)$-dimensional and all the points in $\phi(\cX)$ belong to $S$ almost surely, i.e., $\prob(\phi(\cX)\cap S^c) = 0$. However, since $\Sigma$ is positive definite (with all of its eigenvalues target than $C_{\min}$, by Assumption~\ref{ass1}), $\prob_X(S)= 0$. Combining these observations, $\prob(\phi(\cX)) \le \prob(S) + \prob(\phi(\cX)\cap S^c) = 0$. Since $\phi$ has the $0$-property, this implies that $\prob_X(\cX) = 0$, which is a contradiction because $\cX$ is the support of $\prob_X$ and thus $\prob_X(\cX) = 1$.
The result follows.
\subsection{Proof of Lemma~\ref{lem:opt-price}}\label{app:opt-price}
Under model~\eqref{eq:model3}, a purchase occurs at time $t$ with the posted price price $p$ if $\tv_t \ge \beta_0 p$. This is equivalent to $\tz_t \ge \beta_0 p - \tx_t\cdot \p_0$. Therefore, the expected revenue from a posted price $p$ is given by
\begin{align}
p\times \prob(\tv_t\ge \beta_0 p) = p(1-F(\beta_0 p-\p_0\cdot \tx_t))\,.
\end{align}
By setting the first order conditions, the optimal price $p^*_t$ is given by the solution of the following equation:
\begin{align}
\beta_0 p^*_t = \frac{1-F(\beta_0p^*_t - \p_0 \cdot \tx_t)}{ f(\beta_0p^*_t - \p_0\cdot \tx_t)}\,.
\end{align}
It is straightforward to verify that the solution $p^*_t$ of the above equation is given by $p^*_t = (1/\beta_0) g(\p_0 \cdot \tx_t)$.
\subsection{Proof of Proposition~\ref{thm:learning2}} \label{app-learning2}
This proposition can be proved by following similar steps as in proof of Propostion~\ref{thm:learning}. Indeed, in that proof most of the steps hold for any log-concave loss function and, in particular, for the quadratic loss, with $u_{\l1u} = 2(K+\l1u)$ and $\ell_{\l1u} = 2$. The only difference is that in Proposition~\ref{thm:learning}, we had the negative log-likelihood loss function $\cL$ and we used the observation that the expected loss vanishes at the true model parameters. Namely, we had $\nabla \cL(\mu) = (1/n)\sum_{t=1}^n \xi_t(\mu)\tx_t$ and we showed that $\E(\xi_t(\p_0)) = 0$, from which we derived the high probability bound on $\|\nabla \cL(\p_0)\|_\infty$. (See definition of event $\cF$ given by~\eqref{eq:F}).  

We show that a similar property holds for the quadratic loss function~\eqref{quad-loss}. To see this, recall that in the exploration phases the prices are drawn uniformly at random from the interval $[0,K]$. Therefore, $\E(y_t|v_t) = \prob(v_t\ge p_t|v_t) = v_t/K$. Letting $\xi_t (\p) = 2 (Ky_t - \tx_t\cdot \p)$, we have $\nabla \cL(\mu) = {1}/(ck) \sum_{t\in\cA_k}^n \xi_t(\mu)\tx_t$ and
\begin{align}
\E(\xi_t(\p_0)) = 2\E(K\E(y_t|v_t) - \tx_t\cdot \p_0) =2\E(v_t - \tx_t\cdot \p_0) = 2\E(z_t) = 0\,,
\end{align}
where in the fist step, the inner expectation is with respect to price $p_t$. Given that $\{(\tx_t,z_t)\}_{t\ge1}$ are independent across $t$, by applying Azuma-Hoeffding inequality and a union bonding over $d$ coordinates of features, we obtain that 
\begin{align}
\prob\left(\|\nabla \cL(\p_0)\|_\infty \ge 4(K+\l1u) \sqrt{\frac{\log d}{ck}} \right) \ge 1 - \frac{1}{d}\,.
\end{align}
The rest of the proof is similar to the proof of Proposition~\ref{thm:learning} and is omitted.

\subsection{Proof of Lemma~\ref{lem:S-term2}}\label{proof:lem-S-term2}
Recall the notation $\|E^{(k)}\|_\infty = \max_{i,j} |E^{(k)}_{ij}|$. Note that 
\begin{align}
\<v,E^{(k)} v\> &= \sum_{i,j=1}^{d+1} |E^{(k)}_{ij}| |v_i| |v_j| \le \|E^{(k)}\|_\infty \sum_{i,j=1}^{d+1} |v_i|\, |v_j|\nonumber\\
&\le \|E^{(k)}\|_\infty \left( \sum_{i=1}^d |v_i|\right)^2 = \|E^{(k)}\|_\infty \|v\|_1^2\,.\label{eq:S-term2-1}
\end{align}
Therefore, we only need to bound $\|E^{(k)}\|_\infty$. Fix $1  \le i, j\le d+1$. We then have
\begin{align}
E^{(k)}_{ij} = \tSigma^{(k)}_{ij} - S^{(k)}_{ij} = \frac{1}{\tau_{k}} \sum_{\ell=1}^ {\tau_k} \Big\{ \tX_{\ell i} \tX_{\ell j} - \E(\tX_{\ell i} \tX_{\ell j})\Big\}\,.
\end{align} 
Let $u^{(ij)}_{\ell} = \tX_{\ell i} \tX_{\ell j}$. Then, $|u^{(ij)}_\ell| \le1$ because $\|\tx_\ell\|_\infty \le 1$. By applying Hoeffding's inequality,
\begin{align}
\prob\left(|E^{(k)}_{ij}| \ge 3 \sqrt{\frac{\log d}{\tau_k}}\right) \le \frac{2}{d^4}\,.
\end{align}  
Therefore, by union bonding over all indices $1\le i,j\le d+1$, we obtain that $\|E^{(k)}\|_\infty \le 3\sqrt{(\log d)/\tau_k}$, with probability at least $1 - 8/d^2$.
The claim follows from this result along with~\eqref{eq:S-term2-1}.  

\medskip
\medskip
\paragraph{\bf Acknowledgments}
Authors are thankful to Arnoud den Boer and Paat Rusmevichientong for their suggestions that improved this work. A. J. would also like to acknowledge the financial support of the Office of the Provost at the University of Southern California through the Zumberge Fund Individual Grant Program. Authors are supported in part by a Google Faculty Research Award.

\bibliographystyle{amsalpha}
\bibliography{dynamicpricing}
\end{document}